\documentclass[11pt]{article}
\usepackage{setspace}
\usepackage{cprotect}
\usepackage{amsmath,amssymb,amsthm}
\usepackage[noend]{algorithmic}
\usepackage[ruled,vlined]{algorithm2e}
\usepackage{hyperref}
\usepackage{fullpage}

\usepackage{makeidx}
\usepackage{enumerate}
\usepackage{graphicx,float,psfrag,epsfig}
\usepackage{epstopdf}
\usepackage{color}
\usepackage{enumitem}
\usepackage{subfig}
\usepackage{caption}
\usepackage{bigints}
\usepackage{mathtools}
\usepackage[mathscr]{euscript}

\DeclareSymbolFont{rsfs}{U}{rsfs}{m}{n}
\DeclareSymbolFontAlphabet{\mathscrsfs}{rsfs}

\numberwithin{equation}{section}

\newtheoremstyle{myexample} 
    {\topsep}                    
    {\topsep}                    
    {\rm }                   
    {}                           
    {\bf }                   
    {.}                          
    {.5em}                       
    {}  

\newtheoremstyle{myremark} 
    {\topsep}                    
    {\topsep}                    
    {\rm}                        
    {}                           
    {\bf}                        
    {.}                          
    {.5em}                       
    {}  

\newtheorem{claim}{Claim}[section]
\newtheorem{lemma}[claim]{Lemma}

\newtheorem{theorem}{Theorem}
\newtheorem{proposition}[claim]{Proposition}

\theoremstyle{myremark}
\newtheorem{remark}{Remark}[section]

\theoremstyle{myremark}

\theoremstyle{myexample}

\def\hx{\hat{x}}
\def\hu{\hat{u}}

\def\hbx{\hat{\boldsymbol x}}

\def\op{\mbox{\tiny\rm op}}

\def\sT{{\sf T}}
\def\<{\langle}
\def\>{\rangle}

\def\prob{{\mathbb P}}

\def\E{{\mathbb E}} 

\newcommand\norm[1]{\left\lVert{#1}\right\rVert}

\def\de{{\rm d}}

\def\reals{\mathbb{R}}

\def\normal{{\sf N}}
\def\cnormal{{\sf CN}}

\def\sb{{\sf b}}
\def\bsb{\bar{\sf b}}

\def\cT{{\mathcal{T}}}

\def\bg{{\boldsymbol g}}

\def\by{{\boldsymbol y}}

\def\bD{{\boldsymbol D}}

\def\bI{{\boldsymbol I}}

\def\be{{\boldsymbol e}}

\def\bx{{\boldsymbol x}}
\def\bxs{\hat{\boldsymbol x}^{\rm s}}
\newcommand{\bxl}{\hat{\boldsymbol x}^{\rm L}}
\def\b0{{\boldsymbol 0}}

\def\ba{{\boldsymbol a}}

\def\olL{\bar{L}}
\def\olDel{\bar{\Delta}}

\def\bZ{{\boldsymbol Z}}

\def\tbu{\tilde{\boldsymbol u}}

\def\bu{{\boldsymbol u}}
\def\bn{{\boldsymbol n}}
\def\bw{{\boldsymbol w}}

\def\hbx{\hat{\boldsymbol x}}
\def\hbu{\hat{\boldsymbol u}}

\def\tbx{\tilde{\boldsymbol x}}
\def\tX{\tilde{X}}
\def\tU{\tilde{U}}
\def\tx{\tilde{x}}
\def\tu{\tilde{u}}
\def\tf{\tilde{f}}

\def\th{\tilde{h}}

\def\id{{\boldsymbol I}}
\def\bzero{{\boldsymbol 0}}

\def\bA{{\boldsymbol A}}

\def\bg{{\boldsymbol g}}

\def\U{{\rm U}}
\def\U1{{\rm U}(1)}

\def\eps{{\varepsilon}}

\def\tow2{\stackrel{W_2}{\Rightarrow}}
\def\PL{{\rm PL}}

\newcommand\abs[1]{\left\lvert{#1}\right\rvert}

\def\sfp{{\sf p}}
\def\sc{{\sf c}}
\def\bsc{\bar{\sf c}}
\def\tsc{\tilde{\sf c}}
\def\tsb{\tilde{\sf b}}

\newcommand{\beq}{\begin{equation}}
\newcommand{\eeq}{\end{equation}}
\newcommand{\diag}{{\rm diag}}

\title{Approximate Message Passing with Spectral Initialization\\ for Generalized Linear Models}

\author{Marco Mondelli\thanks{Institute of Science and Technology (IST) Austria. Email: \texttt{marco.mondelli@ist.ac.at}.}\;\;\;and\;\;\;Ramji Venkataramanan\thanks{Department of Engineering, University of Cambridge. Email: \texttt{ramji.v@eng.cam.ac.uk}.}}

\begin{document}

\maketitle

\begin{abstract}
 We consider the problem of estimating a signal from measurements obtained via a generalized linear model.  We focus on estimators based on approximate message passing (AMP), a family of iterative algorithms with many appealing features: the performance of AMP in the high-dimensional limit can be succinctly characterized under suitable model assumptions; AMP can also be tailored to the empirical distribution of the signal entries, and for a wide class of estimation problems, AMP is  conjectured to be optimal among all polynomial-time algorithms. 

However, a major issue of AMP is that in many models (such as phase retrieval),  it requires an initialization correlated with the ground-truth signal and independent from the measurement matrix. Assuming that such an initialization is available is typically not realistic.  In this paper, we solve this problem by proposing an AMP algorithm initialized with a spectral estimator. With such an initialization, the standard AMP analysis fails since the spectral estimator depends in a complicated way on the design matrix.   Our main  contribution is a rigorous characterization of the performance of AMP with spectral initialization in the high-dimensional limit. The key technical idea is to  define and analyze a two-phase artificial AMP algorithm that first produces the spectral estimator, and then closely approximates the iterates of the true AMP.   We also provide numerical results that demonstrate the validity of the proposed approach. 
\end{abstract}

\section{Introduction}

We consider the problem of estimating a $d$-dimensional signal $\bx \in\mathbb R^d$ from $n$ i.i.d. measurements $y_i\sim p(y\mid \langle\bx, \ba_i \rangle)$, $i\in\{1, \ldots, n\}$, where $\langle\cdot, \cdot\rangle$ is the scalar product, $\{\ba_i\}_{1\le i\le n}$ are given sensing vectors, and the (stochastic) output function $p(\cdot \mid \langle\bx, \ba_i \rangle)$ is a given probability distribution. This is known as a \emph{generalized linear model}  \cite{mccullagh2018generalized}, and it encompasses many settings of interest in statistical estimation and signal processing \cite{noiseRangan, boufounos20081, photonVetterli, eldar2012compressed}. One notable example is the problem of phase retrieval \cite{phFienup, shechtman2015phase}, where
$y_i = |\langle\bx, \ba_i\rangle|^2+w_i$,
with $w_i$ being noise. Phase retrieval appears in several areas of science and engineering, see e.g. \cite{fienup1987phase, millane1990phase, demanet2017convex}, and the last few years have witnessed a surge of interest in the design and analysis of efficient algorithms; see the review \cite{fannjiang2020numerics} and the discussion at the end of this section. 

Here, we consider generalized linear models (GLMs) in the high-dimensional setting where $n,d\to\infty$, with their ratio tending to a fixed constant, i.e., $n/d\to \delta\in \mathbb R$. We focus on a family of iterative algorithms known as approximate message passing (AMP). AMP algorithms are derived via approximations of  belief propagation on the factor graph representing the estimation problem. AMP algorithms were first proposed for estimation in linear models \cite{DMM09, BM-MPCS-2011}, and  for estimation in GLMs  \cite{RanganGAMP}. AMP has since been applied to a wide range of high-dimensional statistical estimation problems including compressed sensing \cite{krzakala2012, BayatiMontanariLASSO, maleki2013asymptotic}, low rank matrix estimation \cite{rangan2012iterative,deshpande2014information, kabashima2016phase}, group synchronization \cite{perry2018message}, and specific instances of GLMs such as logistic regression \cite{sur2019modern} and phase retrieval \cite{schniter2014compressive,ma2019optimization, maillard2020phase}. 

Starting from an initialization $\bx^0 \in \reals^d$, the AMP algorithm for GLMs produces iteratively refined estimates of the signal, denoted by $\bx^t$, for $t \geq 1$.
An appealing feature of AMP is that, under suitable model assumptions, its performance in the high-dimensional limit can be precisely characterized by a succinct deterministic recursion called \emph{state evolution} \cite{BM-MPCS-2011, bolthausen2014iterative, javanmard2013state}.  Specifically, in the high-dimensional limit, the empirical distribution of the estimate $\bx^t$ converges to the law of the random variable $\mu_t X  + \sigma_t W_t$, for $t \ge 1$. Here $X \sim P_X$ (the signal prior), and $W_t \sim  \normal(0,1)$ is independent of $X$. The state evolution recursion specifies how the constants $(\mu_t, \sigma_t)$ can be computed from  $(\mu_{t-1}, \sigma_{t-1})$ (see Sec. \ref{sec:GAMP_spec} for details).

Using the state evolution analysis, it has been shown that AMP provably achieves Bayes-optimal performance in some special cases \cite{DonSpatialC13,deshpande2014information, montanari2017estimation}. Indeed, a conjecture from statistical physics posits that AMP is optimal among all polynomial-time algorithms. The optimality of AMP for generalized linear models is discussed in \cite{barbier2019optimal}.

 However,  when used for estimation in GLMs,  a major issue in current AMP theory is that in many problems (including phase retrieval) we require an initialization $\bx^0$ that is correlated with the unknown signal $\bx$ but independent of the sensing vectors $\{ \ba_i \}$. It is often not realistic to assume that such a realization is available.
 For such GLMs,  without a correlated initialization, asymptotic state evolution analysis predicts that the AMP estimates will be uninformative, i.e., their normalized correlation with the signal vanishes in the large system limit. Thus, developing an AMP theory that does not rely on unrealistic assumptions about the initialization is an important open problem.
  
 In this paper, we solve this open problem for a wide class of GLMs by rigorously analyzing the AMP algorithm with a \emph{spectral estimator}. The idea of using a spectral estimator for GLMs was introduced in  \cite{li1992principal}, and its performance in the high-dimensional limit was recently characterized in \cite{lu2017phase,mondelli2017fundamental}. It was shown that the normalized correlation of the spectral estimator with the signal undergoes a phase transition, and for  the special case of phase retrieval, the  threshold for strictly positive correlation with the signal matches the information-theoretic threshold \cite{mondelli2017fundamental}.

 Our main technical contribution is a novel analysis of AMP with spectral initialization for GLMs, under the assumption that the sensing vectors $\{\ba_i\}$ are i.i.d. Gaussian. This yields a rigorous  characterization of the performance in the high-dimensional limit (Theorem \ref{thm:GLM_spec}).  
 The analysis of AMP with spectral initialization is far from obvious since the spectral estimator depends in a non-trivial way on the sensing vectors $\{\ba_i\}$. The existing state evolution analysis for GLMs \cite{RanganGAMP, javanmard2013state} crucially depends on the AMP initialization being independent of the sensing vectors, and therefore cannot be directly applied.

 At the center of our approach is the design and analysis of an \emph{artificial AMP} algorithm.  The artificial AMP operates in two phases: in the first phase, it performs a power method, so that its iterates approach the spectral initialization of the true AMP; in the second phase, its iterates are designed to remain close to the iterates of the true AMP. The initialization of the artificial AMP is correlated with $\bx$, but independent of the sensing vectors $\{\ba_i\}$, which allows us to apply the standard state evolution analysis.  Note that the initialization of the artificial AMP is impractical (it requires the knowledge of the unknown signal $\bx$!). However, this is not an issue, since the artificial AMP is employed solely as a proof technique: we prove a state evolution result for the true AMP by showing that its iterates are close to those in the second phase of the artificial AMP.

 Initializing AMP with a (different) spectral method has been recently shown to be  effective for low-rank matrix estimation \cite{montanari2017estimation}.  However,  our proof technique for analyzing spectral initialization for GLMs is different from the approach in  \cite{montanari2017estimation}. The argument in that paper is specific to the spiked random matrix model and relies on a delicate decoupling argument between the outlier eigenvectors and the bulk. Here, we follow an approach developed in \cite{mondelli2020optimal}, where a specially designed AMP is used to establish the joint empirical distribution of the  signal, the spectral estimator, and the linear estimator.

 For the case of phase retrieval, in \cite{MaXM18} it is provided a heuristic argument for the validity of spectral initialization, and it is stated that establishing a rigorous proof is an open problem. Our paper not only solves this open problem, but it also gives a provable initialization method valid for a class of GLMs. 
 
 We note that, for some GLMs, AMP does not require a special  initialization that is correlated with the signal $\bx$. In Section \ref{sec:GAMP_spec}, we give a condition on the GLM output function that specifies precisely when such a correlated initialization  is required  (see \eqref{eq:ncond}).  This condition is satisfied by several popular GLMs, including phase retrieval. It is in these cases that  AMP with spectral initialization is most useful.

\vspace{1em}

\noindent\textbf{Other related work.} For the problem of phase retrieval,  several algorithmic solutions have been proposed and analyzed in recent years. An inevitably non-exhaustive list includes semi-definite programming relaxations \cite{candes2013phaselift, candes2015phase,
  candes2015phase2,  waldspurger2015phase}, a convex relaxation operating in the natural domain of the signal \cite{goldstein2018phasemax, bahmani2017phase}, alternating minimization \cite{netrapalli2013phase}, Wirtinger Flow \cite{candes2015wirt, chen2017solving, ma2020implicit}, iterative projections \cite{li2015phase}, the Kaczmarz method
\cite{wei2015solving, tan2019phase}, and mirror descent \cite{WuRebes2020}. A generalized AMP (GAMP) algorithm was introduced in \cite{schniter2014compressive}, and an AMP to solve the non-convex problem with $\ell_2$ regularization was proposed and analyzed in \cite{ma2019optimization}. Most of the algorithms mentioned above require an initialization correlated with the signal $\bx$ and, to obtain such an initialization, spectral methods are widely employed. 


Beyond the Gaussian setting, spectral methods for phase retrieval with random orthogonal matrices are analyzed in \cite{dudeja2020analysis}.
Statistical and computational phase transitions in phase retrieval for a large class of correlated real and complex random sensing matrices are investigated in \cite{maillard2020phase}, and a general AMP algorithm for rotationally invariant matrices is studied in \cite{fan2020approximate}. In \cite{emami2020generalization}, it is characterized the generalization error of GLMs via the multi-layer vector AMP (ML-VAMP) of \cite{fletcher2018inference,pandit2020inference}. Thus, the extension of our techniques to more general sensing models represents an interesting avenue for future research.



%
%
%
%



\section{Preliminaries}\label{sec:prel}

\noindent \textbf{Notation and definitions.} Given $n\in \mathbb N$, we use the shorthand $[n]=\{1, \ldots, n\}$. Given a vector $\bx$, we denote by $\| \bx \|_2$  its Euclidean norm. The \emph{empirical distribution} of a vector $\bx = (x_1, \ldots, x_d)^{\sT}$ is given by $ \frac{1}{d}\sum_{i=1}^d \delta_{x_i}$, where $\delta_{x_i}$ denotes a Dirac delta mass on $x_i$. Similarly, the empirical joint distribution of vectors $\bx, \bx'\in \reals^d$ is $\frac{1}{d} \sum_{i=1}^d \delta_{(x_i, x'_i)}$.

\vspace{1em}

\noindent \textbf{Generalized linear models.} Let $\bx \in \reals^d$ be the signal of interest, and assume that $\|\bx\|^2_2=d$.  The signal is observed via inner products with $n$ sensing vectors $(\ba_i)_{i \in [n]}$, with each $\ba_i  \in \reals^d$  having independent Gaussian entries with mean zero and variance $1/d$, i.e., $( \ba_i) \ \sim_{\rm i.i.d.} \ \normal(0, \bI_d/d)$. Given $g_i = \< \bx, \ba_i \>$, the components of the observed vector $\by=(y_1, \ldots, y_n) \in \reals^n$ are independently generated according to a conditional distribution $p_{Y| {G}}$, i.e., $y_i \ \sim \ p_{Y | G}(y_i \mid {g}_i)$.
We stack the sensing vectors as rows to define the $n  \times d$ sensing matrix $\bA$, i.e., $\bA = [\ba_1, \ldots, \ba_n]^{\sT}$.
For the special case of phase retrieval, the model is $\by = \abs{\bA \bx}^2 + \bw$, where $\bw$ is a noise vector with independent entries. We consider a sequence of problems of growing dimension $d$, and assume that, as $d \to \infty$, the sampling ratio $n/d \to \delta$, for some constant $\delta \in (0, \infty)$. We remark that, as $d\to\infty$, the empirical distribution of $\bg=(g_1, \ldots, g_n)$ converges in Wasserstein distance ($W_2$) to $G\sim\normal(0, 1)$.

\vspace{1em}

\noindent \textbf{Spectral initialization.} The spectral estimator $\bxs$ is the principal eigenvector of the $d \times d$ matrix $\bD_n$, defined as 
\beq
\label{eq:Dn_def}
\bD_n =   \bA^\sT \bZ_s \bA, \mbox{ with }\bZ_s = \diag(\cT_s(y_1), \ldots, \cT_s(y_n)),
\eeq
where $\cT_s: \reals \to \reals$ is a preprocessing function.  We now review some results from \cite{mondelli2017fundamental,lu2017phase} on the performance of the spectral estimator  in the high-dimensional limit.

Let $G\sim \normal(0, 1)$, $Y\sim p(\cdot \mid G)$, and $Z_s=\mathcal T_s(Y)$. We will make the following assumptions on $Z_s$. 

\begin{enumerate}
    \item[\textbf{(A1)}] $\mathbb P(Z_s=0)<1$. \label{assump:A1}
    
    \item[\textbf{(A2)}] $Z_s$ has bounded support and $\tau$ is the supremum of this support:\label{assump:A2}
\begin{equation}
  \tau = \inf\{ z : \mathbb P (Z_s\le z) = 1\}.  
\end{equation}

    \item[\textbf{(A3)}] As $\lambda$ approaches $\tau$ from the right, we have
\begin{equation}\label{eq:hplemmaub1}
\lim_{\lambda\to \tau^+}{\mathbb E}\left\{\frac{Z_s}{(\lambda-Z_s)^2}\right\}=\lim_{\lambda\to \tau^+}{\mathbb E}\left\{\frac{Z_s\cdot G^2}{\lambda-Z_s}\right\}=\infty.
\end{equation}
\end{enumerate}

For $\lambda\in (\tau, \infty)$ and $\delta\in (0, \infty)$, define
\begin{equation}\label{eq:newdef}
\phi(\lambda) = \lambda {\mathbb E}\left\{\frac{Z_s\cdot G^2}{\lambda-Z_s}\right\}, \hspace{1em}
\psi_{\delta}(\lambda) = \frac{\lambda}{\delta}+\lambda{\mathbb E}\left\{\frac{Z_s}{\lambda-Z_s}\right\}.
\end{equation}
Note that $\phi(\lambda)$ is a monotone non-increasing function and that $\psi_{\delta}(\lambda)$ is a convex function. Let $\bar{\lambda}_{\delta}$ be the point at which $\psi_{\delta}$ attains its minimum, i.e.,
$\bar{\lambda}_\delta = \arg\min_{\lambda\ge \tau} \psi_{\delta}(\lambda)$. For $\lambda\in (\tau, \infty)$, also define
\begin{equation}\label{eq:defzeta}
\zeta_{\delta}(\lambda) = \psi_{\delta}(\max(\lambda, \bar{\lambda}_\delta)).
\end{equation}
We remark that $\zeta_{\delta}$ is an increasing function and, from Lemma 2 in \cite{mondelli2017fundamental}, we have that the equation $\zeta_{\delta}(\lambda) = \phi(\lambda)$ admits a unique solution for $\lambda > \tau$.

The following result characterizes the performance of the spectral estimator $\bxs$. Its proof follows directly from Lemma 2 in \cite{mondelli2017fundamental}.

\begin{lemma}\label{lemma:pt}
Let $\bx$ be such that $\|\bx\|^2_2=d$, $\{\ba_i\}_{i\in [n]}\sim_{\rm i.i.d.}\normal({\b0}_d,\id_d/d)$, and $\by=(y_1, \ldots, y_n)$ with $\{y_i\}_{i\in [n]}\sim_{\rm i.i.d.} p_{Y | G}$. Let $n/d\to \delta$, $G\sim \normal(0, 1)$ and define $Z_s=\mathcal T_s(Y)$
for $Y\sim p_{Y|G}$. Assume that $Z_s$ satisfies the assumptions \textbf{(A1)}-\textbf{(A2)}-\textbf{(A3)}. Let $\bxs$ be the principal eigenvector of the matrix $\bD_n$ defined  in \eqref{eq:Dn_def}, and let $\lambda_{\delta}^*$ be the unique solution of $\zeta_{\delta}(\lambda) = \phi(\lambda)$ for $\lambda > \tau$. Then, as $n \to \infty$,
	\begin{equation}\label{eq:numpred}
\frac{|\langle\bxs, \bx\rangle|^2}{\norm{\bxs}_2^2 \, \norm{\bx}_2^2}\hspace{-.1em} \stackrel{\mathclap{\mbox{\footnotesize a.s.}}}{\longrightarrow}\hspace{-.1em}  a^2 
\hspace{-.1em}\triangleq \hspace{-.1em} \left\{\begin{array}{ll}
\vspace{1em}
0, & \hspace{-.75em}\mbox{if }\, \psi_{\delta}'(\lambda_{\delta}^*)\le 0,\\
\hspace{-.5em}\displaystyle\frac{\psi_{\delta}'(\lambda_{\delta}^*)}{\psi_{\delta}'(\lambda_{\delta}^*)-\phi'(\lambda_{\delta}^*)}, &\hspace{-.75em} \mbox{if } \, \psi_{\delta}'(\lambda_{\delta}^*)> 0,\\
\end{array}\right.
	\end{equation}		
where $\psi_{\delta}'$ and $\phi'$ are the derivatives of the respective functions. 
\end{lemma}

\begin{remark}[Equivalent characterization]
Using the definitions \eqref{eq:newdef}-\eqref{eq:defzeta}, the conditions $\zeta_{\delta}(\lambda_{\delta}^*) = \phi(\lambda_{\delta}^*)$ and 
$\psi_{\delta}'(\lambda_{\delta}^*) > 0$ are equivalent to
\begin{equation}
\E\left\{ \frac{Z_s (G^2-1)}{\lambda_{\delta}^* - Z_s} \right\} = \frac{1}{\delta}, \,\,\mbox{and}\,\,\,\,    \E\left\{ \frac{Z_s^2}{(\lambda_{\delta}^* - Z_s)^2} \right\} < \frac{1}{\delta} \,.
\label{eq:corr_cond}
\end{equation}
When these conditions are satisfied, the limit of the normalized correlation in \eqref{eq:numpred} can be expressed as 
\begin{equation}
    a^2 = \frac{\frac{1}{\delta} -  \E\left\{ \frac{Z_s^2}{(\lambda_{\delta}^* - Z_s)^2} \right\} }{ \frac{1}{\delta} + \E\left\{ \frac{Z_s^2(G^2-1)}{(\lambda_{\delta}^* - Z_s)^2} \right\} }.
    \label{eq:a2_explicit}
\end{equation}
\label{rem:a2_explicit}
\end{remark}

\begin{remark}[Optimal preprocessing function]\label{rmk:opt}
In \cite{mondelli2017fundamental} it is derived the  preprocessing function minimizing the value of $\delta$ necessary to achieve weak recovery, i.e., a strictly positive correlation between $\bxs$ and $\bx$. In particular, let $\delta_{\rm u}$ be defined as 
\begin{equation}\label{eq:defdeltaur}
\delta_{\rm u} = \Bigg( \displaystyle\bigintssss_{\mathbb R}\frac{\left({\mathbb E}_{G}\left\{p(y\mid G)(G^2-1)\right\}\right)^2}{{\mathbb E}_{G}\left\{p(y\mid G)\right\}} \,{\rm d}y \Bigg)^{-1},
\end{equation}
with $G\sim \normal(0, 1)$. Furthermore, let us also define
\begin{equation}\label{eq:deftydeltar}
\bar{\mathcal T}(y) = \frac{\sqrt{\delta_{\rm u}}\cdot\mathcal T^*(y)}{\sqrt{\delta}-(\sqrt{\delta}-\sqrt{\delta_{\rm u}})\mathcal T^*(y)},
\end{equation}
where
\begin{equation}\label{eq:deftyr}
\mathcal T^*(y) = 1-\frac{{\mathbb E}_{G}\left\{p(y\mid G)\right\}}{{\mathbb E}_{G}\left\{p(y\mid G)\cdot G^2\right\}}.
\end{equation}
Then, by taking $\cT_s=\bar{\mathcal T}$, for any $\delta>\delta_{\rm u}$, we  almost surely have
\begin{equation}\label{eq:corrub}
\lim_{n\to \infty}\frac{|\langle\bxs, \bx\rangle|}{\norm{\bxs}_2 \, \norm{\bx}_2} > \epsilon,
\end{equation}
for some $\epsilon >0$. Furthermore, for any $\delta < \delta_{\rm u}$, there is no pre-processing function $\mathcal T$ such that, almost surely, \eqref{eq:corrub} holds. For a more formal statement of this result, see Theorem 4 of \cite{mondelli2017fundamental}.
The preprocessing function that, at a given $\delta>\delta_{\rm u}$, maximizes the correlation between $\bxs$ and $\bx$ is also related to $\mathcal T^*(y)$ as defined in \eqref{eq:deftyr},  and it is derived in \cite{luo2019optimal}.
\end{remark}
\section{Generalized Approximate Message Passing with Spectral Initialization} \label{sec:GAMP_spec}

We make the following additional assumptions on the signal $\bx$, the output distribution $p_{Y\mid G}$, and the preprocessing function $\cT_s$ used for the spectral estimator.

\begin{enumerate}
    \item[\textbf{(B1)}] Let $\hat{P}_{X,d}$ denote the \emph{empirical distribution} of $\bx \in \reals^d$. As $d \to \infty$, $\hat{P}_{X,d}$ converges weakly to a distribution $P_X$ such that $\lim_{d \to \infty} \E_{\hat{P}_{X,d}}\{ \abs{X}^{2} \} = \E_{P_X}\{ \abs{X}^{2} \}$. We note that $\E_{P_X}\{\abs{X}^2\} =1$, since we assume  $\|\bx\|_2^2=d$. \label{assump:xdist}
    
    \item[\textbf{(B2)}] We have $\E\{ \abs{Y}^{2} \} < \infty$, for $Y\sim p_{Y|G}(\,\cdot\,|\,G)$ and $G \sim \normal(0,1)$. Furthermore, there exists a function  $q: \reals \times \reals \to \reals$  and a random variable $V$ independent of  $G$ such that $Y= q(G, V)$. More precisely, 
for any measurable set $A \subseteq \mathcal{Y}$ and almost every $g$, we have $\prob( Y \in A \mid G =g ) = \prob(q(g, V) \in A)$. We also assume that $\E\{ \abs{V}^{2} \} < \infty$. This is without loss of generality due to the functional representation lemma, see p. 626 of \cite{el2011network}. 
\label{assump:h_zv}

    \item[\textbf{(B3)}] The function $\cT_s: \reals \to \reals$ is bounded and Lipschitz.
\end{enumerate}

Following the terminology of \cite{RanganGAMP}, we refer to the AMP for generalized linear models as GAMP.
In each iteration $t$, the proposed GAMP algorithm produces an estimate $\bx^t$ of the signal $\bx$. The algorithm is defined in terms of a sequence of Lipschitz functions $f_t:\mathbb R\to \mathbb R$ and $h_t:\mathbb R \times \reals \to \mathbb R$, for $t \geq 0$. We initialize using the spectral estimator $\bxs$:
\begin{align}
& \bx^0= \sqrt{d} \, \, \frac{1}{\sqrt{\delta}}\bxs, \label{eq:x0_spec} \\
& \bu^0 = \frac{1}{\sqrt{\delta}} \bA f_{0}(\bx^0) - \sb_0 \frac{\sqrt{\delta}}{\lambda_\delta^*} \,  \bZ_s  \bA \bx^0, \label{eq:u0_spec} 
\end{align}
where $\sb_0 = \frac{1}{n}\sum_{i=1}^d f_{0}'(x_i^0)$,  the diagonal matrix $\bZ_s$ is defined in \eqref{eq:Dn_def}, and $\lambda_\delta^*$ is given by \eqref{eq:corr_cond}.  Then, for $t \geq 0$, the algorithm computes:
\begin{align}
\bx^{t+1} &= \frac{1}{\sqrt{\delta}}\bA^\sT h_t(\bu^t; \by) - \sc_t f_t(\bx^t),  \label{eq:xt_update} \\
\bu^{t+1} &= \frac{1}{\sqrt{\delta}}\bA f_{t+1}(\bx^{t+1})-\sb_{t+1} h_t(\bu^{t}; \by). \label{eq:ut_update} 
\end{align}
 Here the functions $f_t$ and $h_t$ are understood to be applied component-wise, i.e., $f_t(\bx^t)=(f_t(x^t_1)$, $\ldots, f_t(x^t_d))$ and $h_t(\bu^t; \by)=(h_t(u^t_1; y_1), \ldots, h_t(u^t_n; y_n))$. The scalars $\sb_t, \sc_t$ are defined as
\begin{equation}
\sc_t = \frac{1}{n}\sum_{i=1}^n h_t'(u_i^t; y_i), \qquad \sb_{t+1} =\frac{1}{n}\sum_{i=1}^d f_{t+1}'(x_i^{t+1}),
\label{eq:GAMP_onsager}
\end{equation}
where $h_t'(\cdot, \cdot)$ denotes the derivative with respect to the first argument.

 The asymptotic empirical distribution of the GAMP iterates $\bx^t, \bu^{t}$, for $t\ge 0$, can be succinctly characterized via a deterministic recursion, called \emph{state evolution}.  Our main result, Theorem \ref{thm:GLM_spec}, shows that for $t \geq 0$, the empirical distributions of $\bu^t$ and $\bx^t$ converge in Wasserstein distance $W_2$ to the laws of the random variables $U_t$ and $X_t$, respectively, with 
\begin{align}
    X_t & \equiv \mu_{X,t} X + \sigma_{X,t} W_{X,t},    \label{eq:Xt_def} \\
        U_t & \equiv \mu_{U,t} G + \sigma_{U,t} W_{U,t},  \label{eq:Ut_def}
\end{align}
where  $(G, W_{U,t}) \sim_{\rm i.i.d.} \normal(0,1)$. Similarly,  $X \sim P_X$ and $W_{X,t} \sim \normal(0,1)$ are independent. The deterministic  parameters $(\mu_{U,t}, \sigma_{U,t}$, $\mu_{X,t}, \sigma_{X,t} )$ are recursively computed as follows, for $t \ge 0$: 
\begin{align}
    \mu_{U,t} & = \frac{1}{\sqrt{\delta}} \E\{ X f_t(X_t) \}, \nonumber \\
     \sigma_{U,t}^2 
     %
     & = 
     \frac{1}{\delta} \E\big\{ f_t(X_t)^2 \big\} - \mu_{U,t}^2, \label{eq:SE_def} \\
     \mu_{X,t+1}&\hspace{-0.15em} =\hspace{-0.15em}\sqrt{\delta} \E\{ G h_t(U_t; Y) \} \hspace{-0.1em}-\hspace{-0.1em}  \E\{ h_t'(U_t;Y) \} \E\{ X f_t(X_t) \}, 
     \nonumber \\
     \sigma_{X,t+1}^2 & = \E\{ h_t(U_t; Y)^2 \}. \nonumber
\end{align}

For the spectral initialization in \eqref{eq:x0_spec}-\eqref{eq:u0_spec},  with $a$ as defined in \eqref{eq:numpred}, the recursion is initialized with
\beq
\label{eq:SE_init}
\mu_{X,0} = {a}/{\sqrt{\delta}}, \qquad \sigma^2_{X,0} = {(1 -a^2)}/{\delta}.
\eeq

We state the main result in terms of \emph{pseudo-Lipschitz} test functions.
A function $\psi: \reals^m \to \reals$ is pseudo-Lipschitz of order $2$, i.e., $\psi \in \PL(2)$, if there is a constant $C > 0$  such that 
\beq
\norm{\psi(\bx)-\psi(\by)}_2 \le C (1 + \| \bx \|_2 + \| \by \|_2 )\norm{\bx - \by}_2,
\label{eq:PL2_prop}
\eeq
for all $\bx,\by \in \mathbb{R}^m$.
Examples of test functions in $\PL(2)$ with $m=2$ include $\psi(a,b) = (a-b)^2$, $\psi(a,b) =ab$. 
\begin{theorem}\label{thm:GLM_spec}
Let $\bx$ be such that $\|\bx\|^2_2=d$, $\{\ba_i\}_{i\in [n]}\sim_{\rm i.i.d.}\normal({\b0}_d,\id_d/d)$, and $\by=(y_1, \ldots, y_n)$ with $\{y_i\}_{i\in [n]}\sim_{\rm i.i.d.} p_{Y | G}$. Let $n/d\to \delta$, $G\sim \normal(0, 1)$, and $Z_s=\mathcal T_s(Y)$ for $Y\sim p_{Y|G}(\,\cdot\,|\,G)$. Assume that \textbf{(A1)}-\textbf{(A2)}-\textbf{(A3)} and \textbf{(B1)}-\textbf{(B2)}-\textbf{(B3)} hold. Assume further that $\psi_{\delta}'(\lambda_{\delta}^*)> 0$, and let $\bxs$ be the principal eigenvector of $\bD_n$, defined as in \eqref{eq:Dn_def}, with the sign of $\bxs$ chosen so that $\langle\bxs, \bx\rangle\ge 0$. 

Consider the GAMP iteration in Eqs.~\eqref{eq:ut_update}--\eqref{eq:xt_update}  with initialization in Eqs. \eqref{eq:x0_spec}-\eqref{eq:u0_spec}. Assume that for $t \geq 0$, the functions $f_t, h_t$ are Lipschitz with derivatives that are continuous almost everywhere.
 Then, the following limits hold almost surely for any \PL($2$) function $\psi:\reals\times\reals\to\reals$ and $t$ such that $\sigma_{X,k}^2$ is  strictly positive for $0 \leq k \leq t$:
\begin{align}
\label{eq:SE_main}
& \lim_{d \to \infty}  \frac{1}{d} \sum_{i=1}^d \psi (x_{i},x^{t+1}_i) \\
& \quad  = \E\{ \psi( X, \, \mu_{X,t+1} X   + \sigma_{X,t+1}  W_{X,t+1}) \}, \nonumber  \\
& \lim_{n \to \infty} \frac{1}{n} \sum_{i=1}^n \psi (y_{i},u^{t}_i) 
 = \E \left\{ \psi( Y, \,  \mu_{U,t} G   + \sigma_{U,t}  W_{U,t}) \right\}. \label{eq:SE_main_U}  
\end{align}
The result \eqref{eq:SE_main} also holds for $(t+1)=0$.
In \eqref{eq:SE_main} (resp. \eqref{eq:SE_main_U}), the expectation is over the independent random variables  $X \sim P_X$ and $W_{X,t} \sim \normal(0,1)$ (resp. $(G, W_{U,t}) \sim_{\rm i.i.d.} \normal(0,1)$). The scalars $(\mu_{X, t}, \mu_{U, t}, \sigma_{X, t}^2, \sigma_{U, t}^2)_{t \geq 0}$ are given by the  recursion \eqref{eq:SE_def} with the initialization \eqref{eq:SE_init}. 
\end{theorem}

We give a sketch of the proof in Section \ref{sec:sketch} and defer the technical details to the appendices.


We now comment on some of the assumptions in the theorem.
The assumption $\psi_{\delta}'(\lambda_{\delta}^*)> 0$ is required to ensure that the spectral initialization $\bx^0$ has non-zero correlation with the signal $\bx$ (Lemma \ref{lemma:pt}). From Remark \ref{rmk:opt}, we also know that for any sampling ratio $\delta>\delta_{\rm u}$ there exists a choice of $\cT_s$ such that $\psi_{\delta}'(\lambda_{\delta}^*)> 0$. We also note  that, for $\delta<\delta_{\rm u}$, GAMP converges to the ``un-informative fixed point'' (where the estimate has vanishing correlation with signal) even if the initial condition has non-zero correlation with the  signal, see Theorem 5 of 
\cite{mondelli2017fundamental}.

There is no loss of generality in assuming the sign of $\bxs$ to be such that $\< \bxs, \bx \> \geq 0$. Indeed, if the sign were chosen otherwise, the theorem would hold with the state evolution initialization in \eqref{eq:SE_init} being $\mu_{X,0} = -a /\sqrt{\delta}, \, \sigma^2_{X,0} = {(1 -a^2)}/{\delta}$.

The assumption that $\sigma^2_{X,k}$ is positive for $k\leq t$ is natural. Indeed, if $\sigma^2_{X,k} = 0$, then the state evolution result for iteration $k$ implies that $\| \bx - \mu_{X,k}^{-1} \bx^k  \|^2/d \to 0$ as $d \to \infty$. That is, we can perfectly estimate $\bx$ from $\bx^k$, and thus terminate the algorithm after iteration $k$.

Let us finally remark that the result in \eqref{eq:SE_main} is equivalent to the statement that the empirical joint distribution of  $(\bx, \bx^{t+1})$ converges almost surely in Wasserstein  distance ($W_2$) to the joint law of $(X, \, \mu_{X, t+1} X + \sigma_{X, t+1} W)$.
This follows from the fact that a sequence of distributions $P_n$ with finite second moment converges in $W_2$  to $P$ if and only if   $P_n$ converges weakly to $P$ and
 $\int \| a \|^2_2 \, \de P_n(a)  \to \int \| a \|^2_2 \, \de P(a)$,  see Definition 6.7 and Theorem 6.8 of
 \cite{villani2008optimal}. 

\vspace{1em}

\noindent \textbf{When does GAMP require spectral initialization?}
For the GAMP to give meaningful estimates, we need  either $\bx^0$ or $\bx^1$ to have strictly non-zero asymptotic  correlation with $\bx$. To see when this can be arranged without a special initialization, consider the \emph{linear} estimator  $\bxl(\xi) :=\bA^\sT \xi(\by)$, for some function $\xi: \reals \to \reals$ that acts component-wise on $\by$. If there exists a function $\xi$ such that the asymptotic normalized correlation between $\bxl(\xi)$ and $\bx$ is strictly non-zero, then AMP does not require a special initialization (spectral or otherwise) that is correlated with $\bx$. Indeed, in this case we can replace the initialization in \eqref{eq:x0_spec}-\eqref{eq:u0_spec} by $\bx^0=\bzero$, $\bu^0 = \bzero$ (by taking $f_0=0$), and let $h_0(\bu^0; \by) =\sqrt{\delta} \xi(\by)$. This gives $\bx^1 = \bA^{\sT} \xi (\by) = \bxl(\xi)$, which has strictly non-zero asymptotic correlation with $\bx$. This  ensures that $| \mu_{X,1}| >0$, and the standard AMP analysis  \cite{javanmard2013state} directly yields a state evolution result similar to Theorem \ref{thm:GLM_spec}.

The output function  $p_{Y|G}$ determines whether a non-trivial linear estimator exists for the GLM. From Appendix C.1 in \cite{mondelli2020optimal}, we have that, if
\begin{equation}\label{eq:ncond}
\displaystyle\bigintssss_{\mathbb R}\frac{\left({\mathbb E}_{G\sim \normal(0, 1)}\left\{G\,p_{Y\mid G}(y\mid G)\right\}\right)^2}{{\mathbb E}_{G\sim \normal(0, 1)}\left\{p_{Y\mid G}(y\mid G)\right\}} \,{\rm d}y=0,
\end{equation}
then the correlation between $\bA^{\sT}\xi(\by)$ and $\bx$ will asymptotically vanish for any choice of $\xi$. The condition \eqref{eq:ncond} holds for many output functions of interest, including all distributions $p_{Y\mid G}$ that are even in $G$ (and, therefore, including phase retrieval). It is for these models that spectral initialization is particularly useful.

We remark that our analysis covers not only the (Wirtinger flow) phase retrieval model $\by =|\bA \bx|^2$, but also the  amplitude flow phase retrieval model $\boldsymbol y = |\boldsymbol A  \boldsymbol x|$. In fact, one can analyze the approximate model $\boldsymbol y = \sqrt{|\boldsymbol A\boldsymbol x|^2 +\boldsymbol \epsilon}$ and then let $\|\boldsymbol \epsilon\|_2 \to 0$. This is similar to the approach taken e.g. by \cite{MaXM18} and \cite{luo2020phase}. Since the functions used in each AMP iteration are Lipschitz, state evolution holds as $\|\boldsymbol \epsilon\|_2 \to 0$. For other GLMs with non-differentiable output functions, we can use a  similar approach to construct a smooth approximation to the output function and obtain the state evolution result.


\vspace{1em}

\noindent \textbf{Bayes-optimal GAMP.} Applying Theorem \ref{thm:GLM_spec} to the \PL(2) function $\psi(x,y)=(x-f_t(y))^2$, we obtain the asymptotic 
 mean-squared error (MSE) of the GAMP estimate $f_t(\bx^t)$. In formulas, for $t \geq 0$, almost surely,
 \beq
 \lim_{d \to \infty} \, \frac{1}{d} \| \bx - f_t(\bx^t) \|_2^2 = 
 \E\{ (X - f_t(\mu_{X,t} X + \sigma_{X,t}W))^2 \}.
 \label{eq:MSE_t}
 \eeq
 If the limiting empirical distribution $P_X$ of the signal is known, then the choice of $f_t$ that minimizes the MSE in \eqref{eq:MSE_t} is 
 \beq
f^*_t(s) = \E\{ X \mid \mu_{X,t} X + \sigma_{X,t}W = s \}. 
\label{eq:ft_opt}
\eeq

Similarly, applying the theorem to the \PL(2) functions $\psi(x,y)=x f_t(y)$ and $\psi(x,y) = f_t(y)^2$, we obtain the asymptotic normalized correlation with the signal. In formulas, for $t \geq 0$, almost surely
\beq
 \lim_{d \to \infty}  \frac{\abs{\< \bx, f_t(\bx^t) \>}}{\| \bx \|_2 \| f_t(\bx^t) \|_2} = 
 \frac{\abs{\E\{ X f_t(\mu_{X,t} X + \sigma_{X,t}W) \}}}
 { \sqrt{\E\{ f_t(\mu_{X,t} X + \sigma_{X,t}W)^2\}}}.
 \label{eq:norm_corr_t}
\eeq
For fixed $(\mu_{X,t}, \sigma_{X,t}^2)$, the normalized correlation in \eqref{eq:norm_corr_t} is maximized by taking $f_t = c f_t^*$ for any $c \neq 0$.
 This choice  also maximizes the ratio 
$\mu_{U,t}^2/\sigma_{U,t}^2$ in \eqref{eq:SE_def}.
For $f_t = cf_t^*$, from \eqref{eq:SE_def} we have
\begin{equation}
\mu_{U,t} = \frac{c}{\sqrt{\delta}} \E\{ f_t^*(X_t)^2 \},
\quad \sigma_{U,t}^2 = \frac{c}{\sqrt{\delta}} \mu_{U,t} -  
\mu_{U,t}^2.
\label{eq:muU_Bopt}
\end{equation}

We now specify the choice of $h_t(u; y)$ that maximizes the ratio $\mu_{X,t+1}^2/\sigma_{X,t+1}^2$ for fixed $(\mu_{U,t}, \sigma^2_{U,t})$.

\begin{proposition}
Assume the setting of Theorem \ref{thm:GLM_spec}. For a given $(\mu_{U,t}, \sigma^2_{U,t})$, the ratio $\mu_{X,t+1}^2/\sigma_{X,t+1}^2$ is maximized when $h_t(u; \, y) =  c \, h_t^*(u; \, y)$  where $c \neq 0$ is any constant, and 
\begin{align}
&  h^*_t(u; \, y)  \triangleq  \frac{\E\{ G \mid U_t= u, \, Y=y \} - 
\E\{ G \mid U_t= u \}}{{\sf Var}(G \mid U_t =u)}
\label{eq:opt_gt} \\
& =   \frac{\E_W\{ W p_{Y|G}( y \, | \, \rho_t u + \sqrt{1-\rho_t\, \mu_{U,t}}\, W) \}}{\sqrt{1-\rho_t\, \mu_{U,t}} \ \E_W\{ p_{Y|G}( y \, | \, \rho_t u + \sqrt{1-\rho_t \, \mu_{U,t}}\, W)\}}, \label{eq:opt_gt_alt2}
\end{align}
where  $\rho_t= \mu_{U,t}/(\mu_{U,t}^2 + \sigma_{U,t}^2)$ and $W \sim \normal(0,1)$.  In \eqref{eq:opt_gt}, the random variables $ U_t$ and $Y$ are conditionally independent given $G$ with 
\beq 
\begin{split}
& Y \sim p_{Y|G}(\, \cdot \, | \, G), \qquad 
U_t = \mu_{U,t} G + \sigma_{U,t} W_{U,t}, \\
& \quad (G, W_{U,t}) \sim_{\rm i.i.d.} \normal(0,1).
\end{split}
\label{eq:GYU_joint}
\eeq
%
\label{prop:opt_gt}
\end{proposition}

The optimal choice for $h^*_t$ in Proposition \ref{prop:opt_gt} was derived by \cite{RanganGAMP} by approximating the belief propagation equations. For completeness, we provide a self-contained proof in Appendix \ref{app:proof_opt_gt}. The proof also shows that with $h_t=c h^*_t$,
\begin{equation*}
\mu_{X,t+1} = c \,  \sqrt{\delta} \, \E\{ \abs{h_t^*(U_t; Y)}^2 \}, \quad 
\sigma_{X,t+1}^2 = c \, \frac{\mu_{X,t+1}}{\sqrt{\delta}}.
\end{equation*}

As the choices $f_t^*, h_t^*$  maximize the signal-to-noise ratios $\mu_{U,t}^2/\sigma_{U,t}^2$ and $\mu_{X,t+1}^2/\sigma_{X,t+1}^2$, respectively, we  refer to this algorithm as Bayes-optimal GAMP. We note that to apply Theorem \ref{thm:GLM_spec} to the Bayes-optimal GAMP, we need $f_t^*$, $h_t^*$ to be Lipschitz. This holds under relatively mild conditions on $P_X$ and $p_{Y|G}$, see Lemma F.1 in \cite{montanari2017estimation}.

\begin{figure}[t!]
    \centering    
    \includegraphics[width=0.65\linewidth]{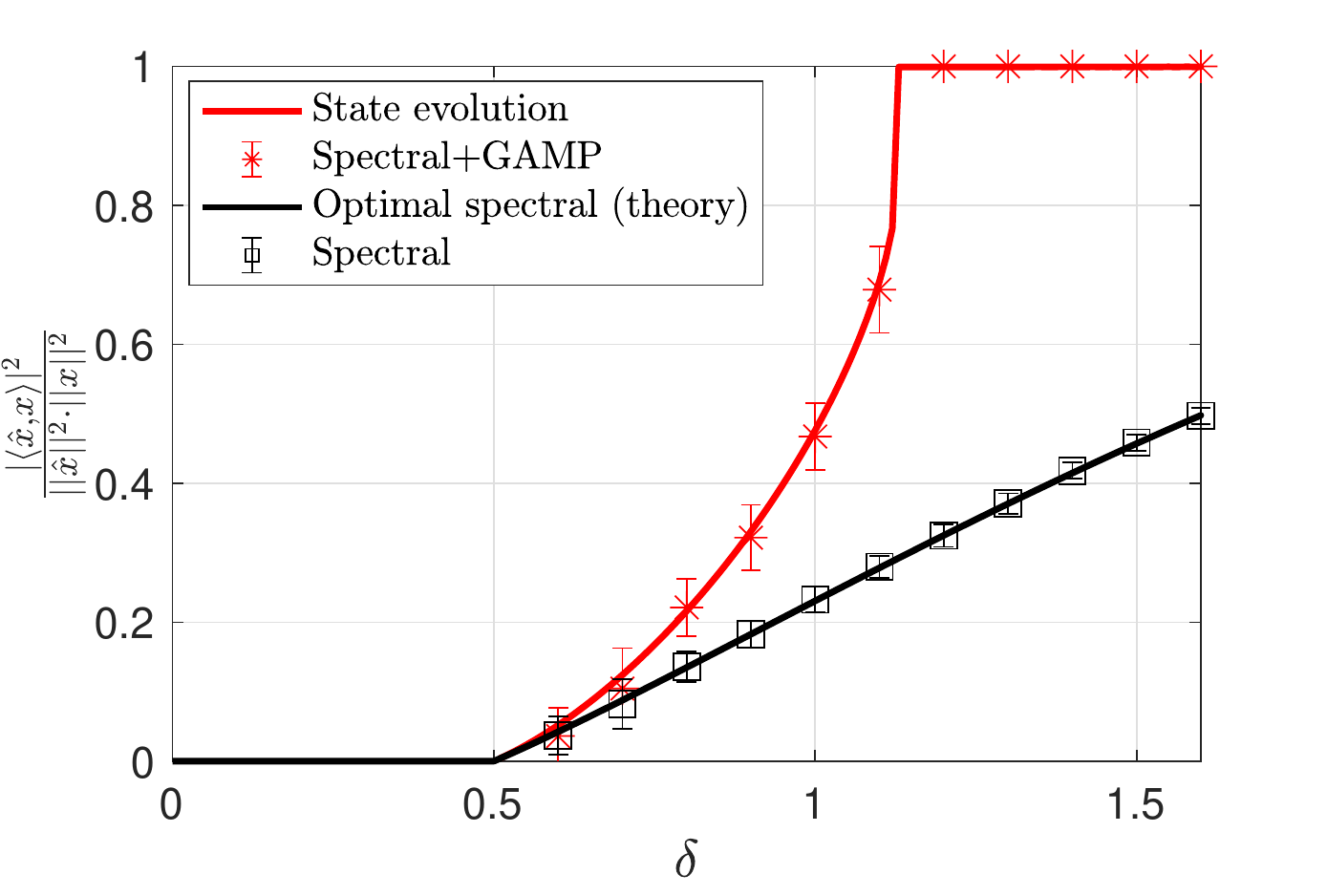}
\caption{Performance comparison between GAMP with spectral initialization (in red) and the spectral method alone (in black) for a Gaussian prior $P_X \sim \normal(0,1)$. The solid lines are the theoretical predictions of Theorem \ref{thm:GLM_spec} for GAMP with spectral initialization, and of Lemma \ref{lemma:pt} for the spectral method. Error bars indicate one standard deviation around the empirical mean.} 
\label{fig:real}
\end{figure}


\section{Numerical Simulations} \label{sec:simulations}

We now illustrate the performance of the GAMP algorithm with spectral initialization via numerical examples. For concreteness, we focus on  noiseless phase retrieval, where $y_i = \abs{\langle \ba_i, \bx\rangle}^2$, $i\in [n]$. 

\vspace{1em}

\noindent \textbf{Gaussian prior.}
In Figure \ref{fig:real}, $\bx$ is chosen uniformly at random on the $d$-dimensional sphere with radius $\sqrt{d}$ and  $\{ \ba_i\}_{i\in [n]} \ \sim_{\rm i.i.d.} \ \normal(0, \bI_d/d)$. Note that, as $d\to \infty$, the limiting empirical distribution $P_X$ of $\bx$ is a standard Gaussian. We take $d=8000$, and the numerical simulations are averaged over $n_{\rm sample} = 50$ independent trials. The performance of an estimate $\hat{\bx}$ is measured via its normalized squared scalar product with the signal $\bx$.
The black points are obtained by  estimating $\bx$ via the spectral method, using the optimal pre-processing function $\mathcal T_s$ reported in Eq. (137) of \cite{mondelli2017fundamental}. The empirical results  match the black curve, which gives the best possible squared correlation in the high-dimensional limit, as given by Theorem 1 of \cite{luo2019optimal}. The red points are obtained by running the GAMP algorithm \eqref{eq:xt_update}-\eqref{eq:ut_update} with the spectral initialization \eqref{eq:x0_spec}-\eqref{eq:u0_spec}. The function $f_t$ is chosen to be the identity, and $h_t =\sqrt{\delta} h_t^*$, for $h_t^*$ given by Proposition \ref{prop:opt_gt}. The algorithm is run until the normalized squared difference between successive iterates  is small.  
As predicted by Theorem \ref{thm:GLM_spec}, the numerical simulations agree well with the state evolution curve in red, which is obtained by computing the fixed point of the recursion \eqref{eq:SE_def} initialized with \eqref{eq:SE_init}. We also remark that the threshold for exact recovery can be obtained from the fixed points of state evolution, see e.g. \cite{barbier2019optimal}.

\begin{figure}[t!]
\centering
\includegraphics[width=0.65\linewidth]{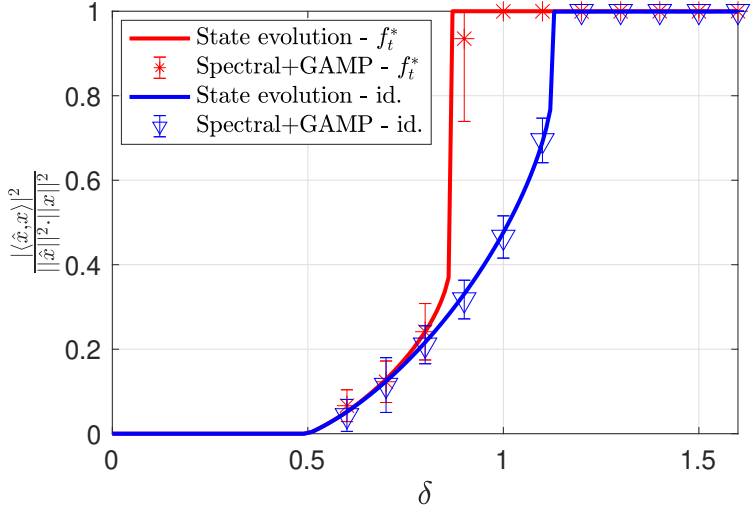}
\caption{Performance comparison between two different choices of $f_t$ for a binary prior $P_X(1) =P_X(-1) =\frac{1}{2}$. The Bayes-optimal choice $f_t=f_t^*$ (in red) has a lower threshold compared to $f_t$ equal to identity (in blue).}\label{fig:Bernprior}
\end{figure}

\begin{figure}[p]
    \centering
    \subfloat[Original image.\label{fig:ven}]{\includegraphics[width=.45\columnwidth]{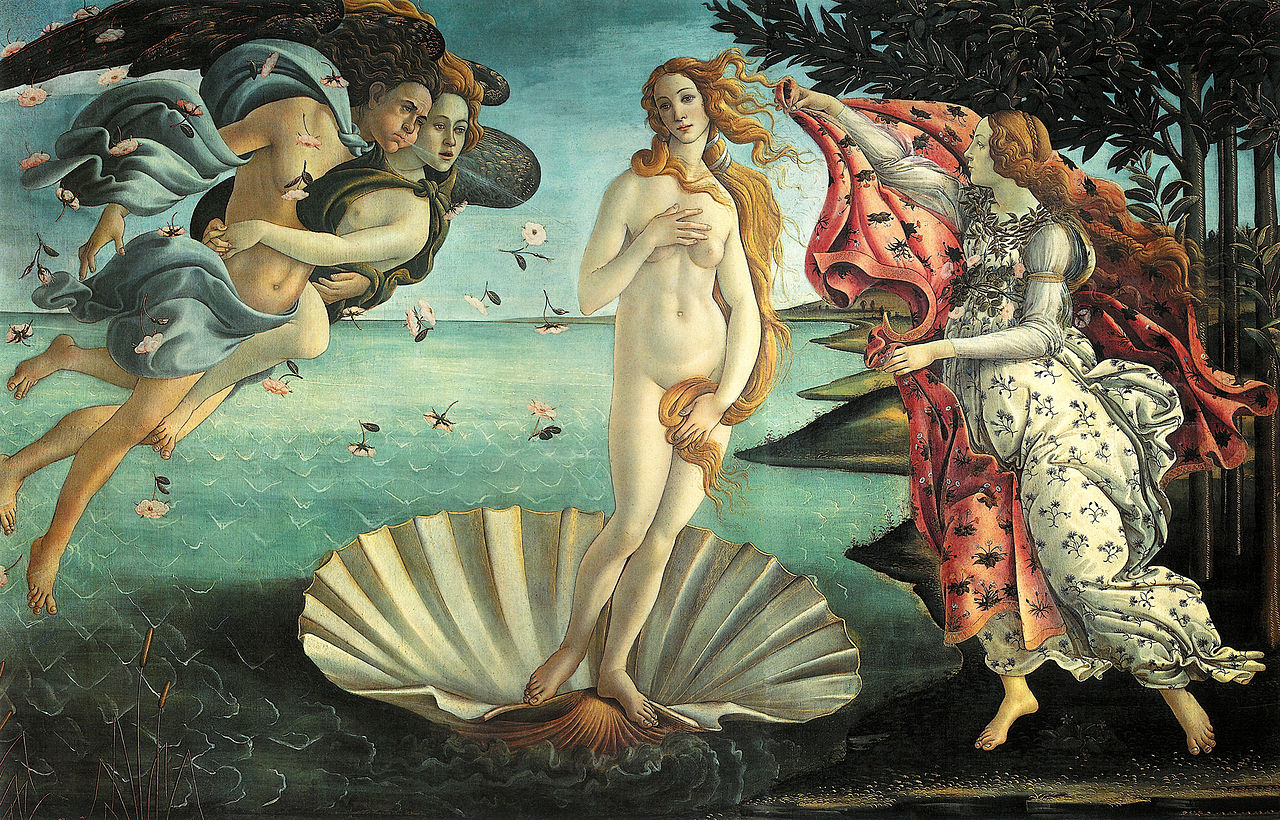}}\\
    \subfloat[Proposed, $\delta=2.2$.\label{fig:venAMP2}]{\includegraphics[width=.45\columnwidth]{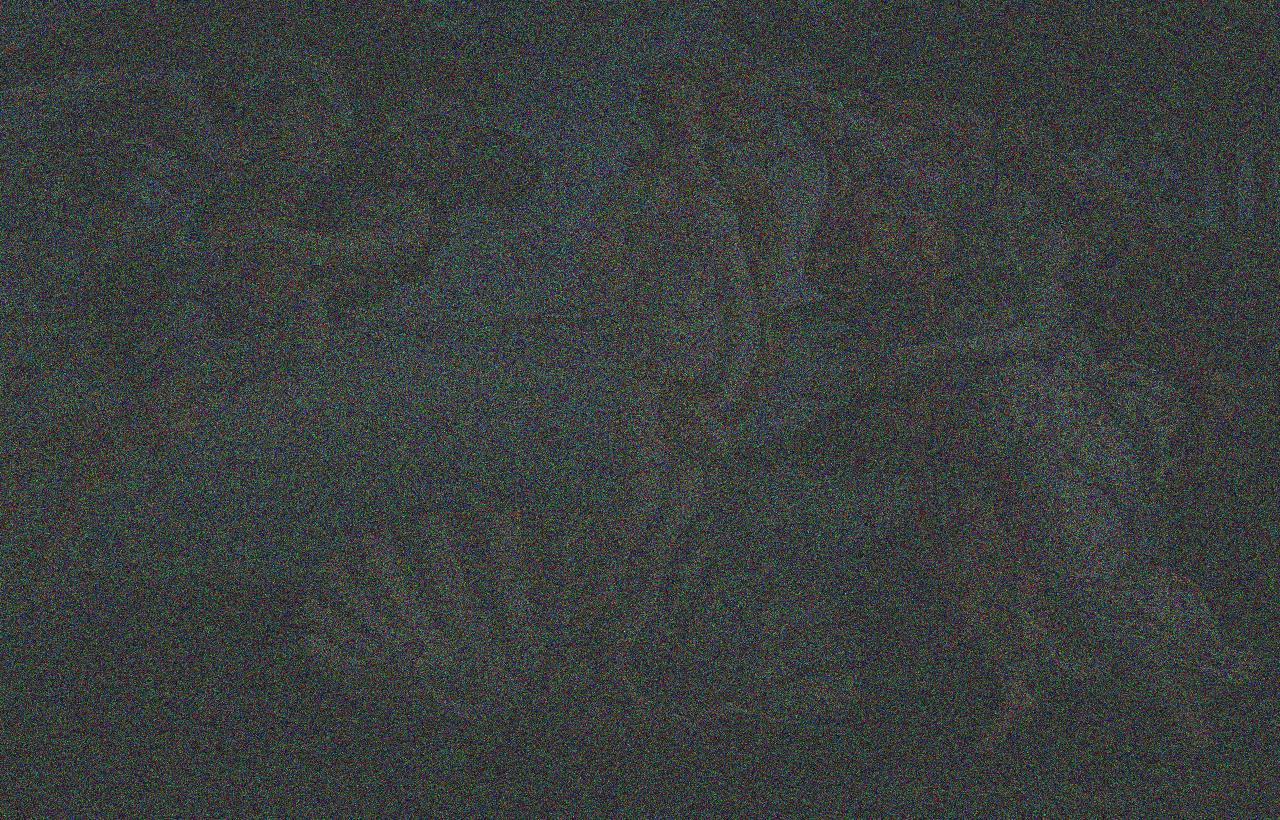}} \hspace{2em}
    \subfloat[Spectral, $\delta=2.2$.\label{fig:veninit2}]{\includegraphics[width=.45\columnwidth]{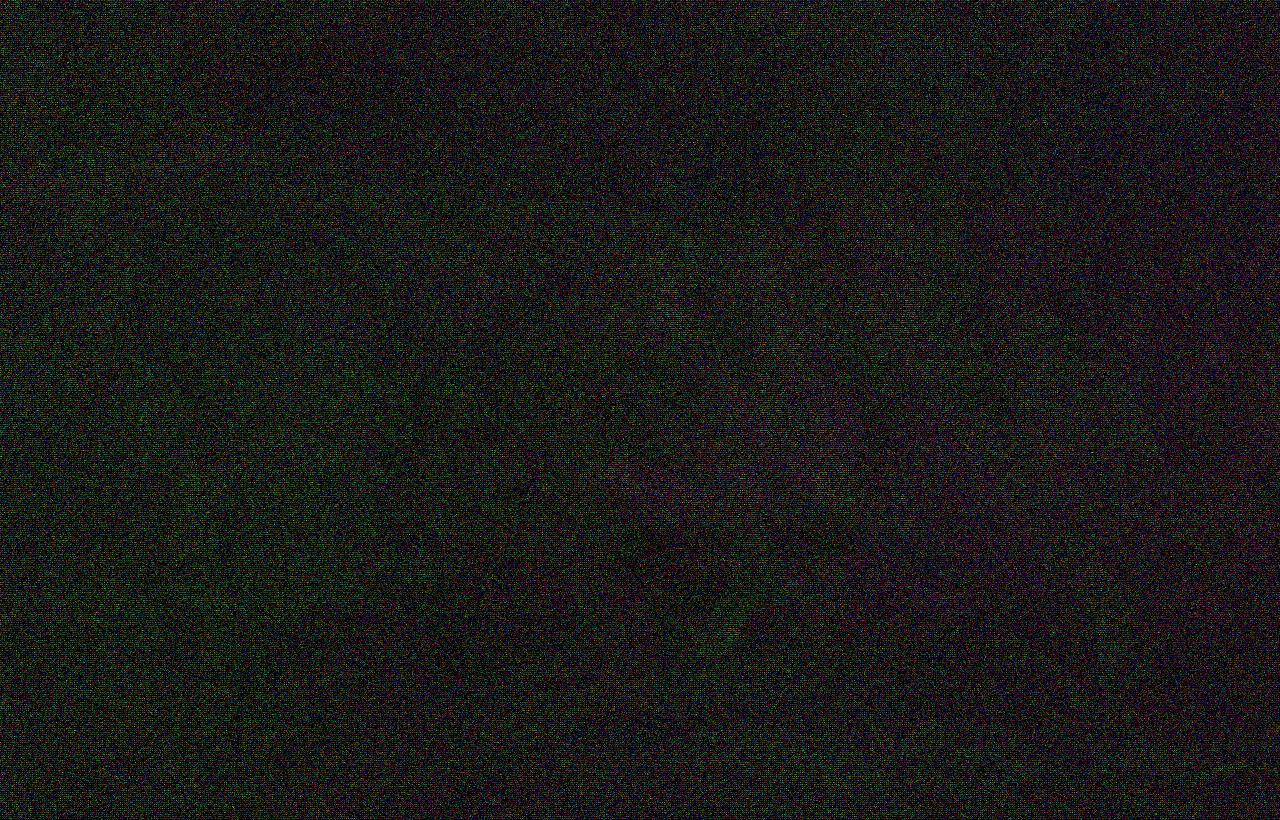}}\\
\subfloat[Proposed, $\delta=2.4$.\label{fig:venAMP4}]{\includegraphics[width=.45\columnwidth]{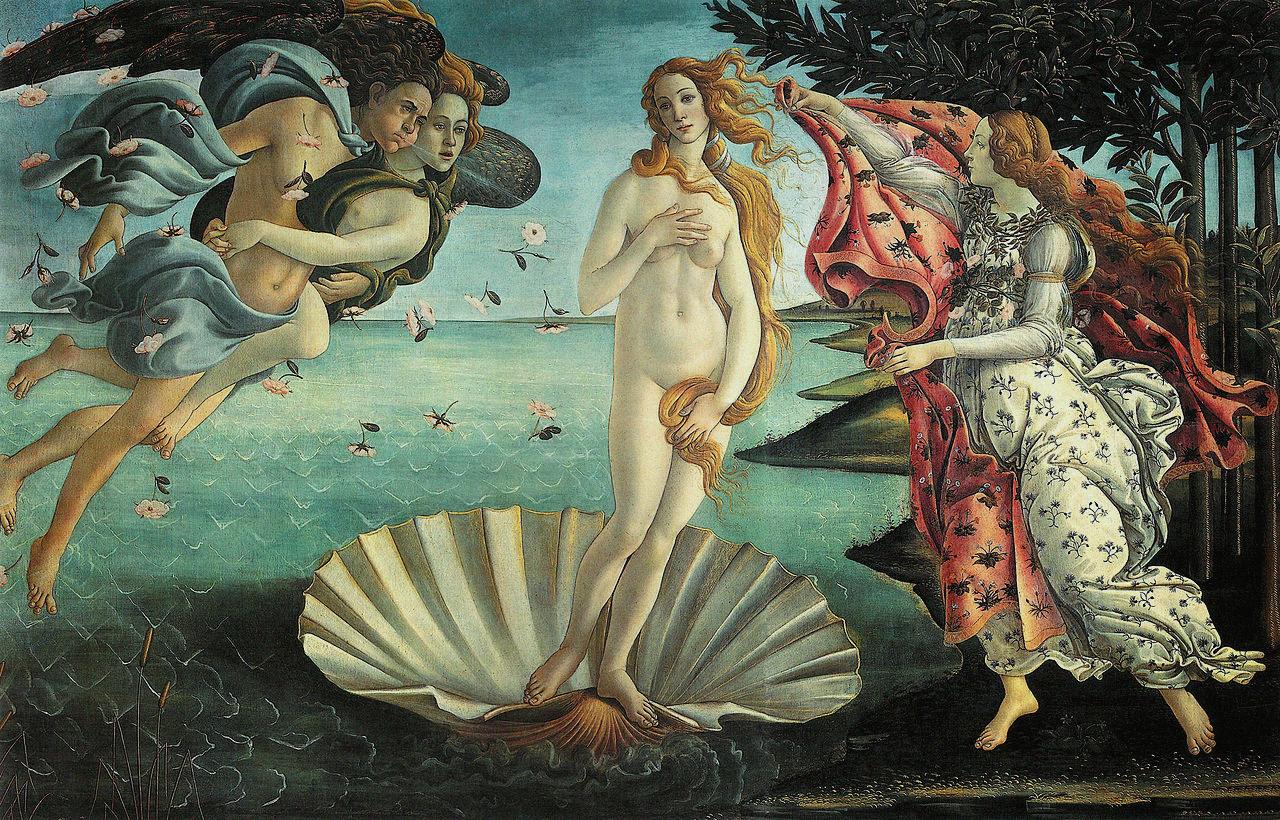}} \hspace{2em}
    \subfloat[Spectral, $\delta=2.4$.\label{fig:veninit4}]{\includegraphics[width=.45\columnwidth]{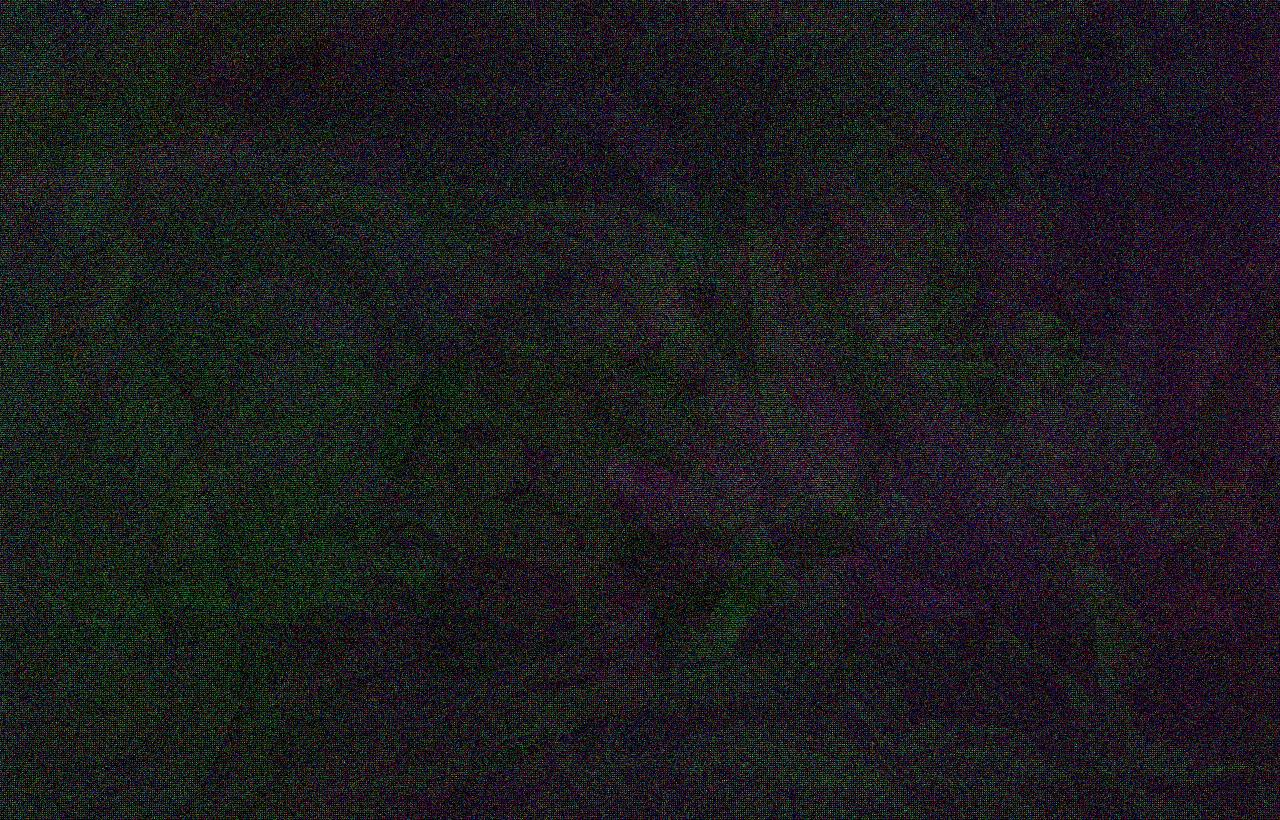}}
\caption{Visual comparison between the reconstruction of the GAMP algorithm with spectral initialization and that of the spectral method alone for measurements given by coded diffraction patterns.}
\label{fig:venrec}
\end{figure}

\vspace{1em}

\noindent \textbf{Bayes-optimal GAMP for a binary-valued prior.} Assume now that each entry of the signal $\bx$ takes value in $\{-1, 1\}$, with $P_X(1) = 1- p_X(-1)= \sfp$. In Figure \ref{fig:Bernprior}, we take $\sfp= \frac{1}{2}$, and compare the performance of the GAMP algorithm with spectral initialization for two different choices of the function $f_t$: $f_t$ equal to identity (in blue) and $f_t = f^*_t$ (in red), where $f_t^*$ is the Bayes-optimal choice \eqref{eq:ft_opt}. By computing the conditional expectation, we have 
\begin{align}
f^*_t(s) & = 2 \prob( X =1 \mid \mu_{X,t} X + \sigma_{X,t} W =s ) -1   = \frac{2 }{1 + \frac{1-\sfp}{\sfp} \exp\big(\frac{-2 s \mu_{X,t}}{\sigma_{X,t}^2} \big)} -1.
\label{eq:ft_binary}
\end{align}
The rest of the setting is analogous to that of Figure \ref{fig:real}. There is a significant performance gap between the Bayes-optimal choice $f_t=f_t^*$ and the choice $f_t(x)=x$. As in the previous experiment, we observe very good agreement between the GAMP algorithm and the state evolution prediction of Theorem \ref{thm:GLM_spec}. 
We  remark that for this setting, the information-theoretically optimal overlap  (computed using the formula  in \cite{barbier2019optimal}) is 1 for all $\delta>0$. Since the components of $\bx$ are in $\{-1,1 \}$, there are  $2^d$ choices for $\bx$. The information-theoretically optimal estimator  picks the choice that is consistent with $y_i = \< \bx, \ba_i \>$, $i \in [n]$. (Since $\bA$ is Gaussian, with high probability this solution is unique.)

\vspace{1em}

\noindent \textbf{Coded diffraction patterns.} \label{subsec:CDP}
We consider the model of coded diffraction patterns described in Section 7.2 of \cite{mondelli2017fundamental}. Here the signal $\bx$ is the image of Figure \ref{fig:ven}, and it can be viewed as a $d_1 \times d_2 \times 3$ array with $d_1 = 820$ and $d_2=1280$. 
The sensing vectors are given by 
\begin{equation}
\begin{split}
 a_r(t_1, t_2) & = d_{\ell}(t_1, t_2)\cdot e^{i2\pi k_1 t_1/d_1}\cdot e^{i2\pi k_2 t_2/d_2},
\end{split}
\label{eq:cdp_def}
\end{equation}
where $r\in [n]$, $t_1\in [d_1]$, $t_2\in [d_2]$, $i$ denotes the imaginary unit, $a_r(t_1, t_2)$ is the $(t_1, t_2)$-th component of $\ba_r\in \mathbb C^{d}$, and the $(d_{\ell}(t_1, t_2))$'s are i.i.d. and uniform in $\{1, -1, i, -i\}$. The index $r\in [n]$ is associated to a pair $(\ell, k)$, with $\ell\in [L]$; the index $k\in [d]$ is associated to a pair $(k_1, k_2)$ with $k_1\in [d_1]$ and $k_2\in [d_2]$. Thus, $n=L\cdot d$ and, therefore, $\delta=L\in \mathbb N$. To obtain non-integer values of $\delta$, we set to $0$ a suitable fraction of the vectors $\ba_r$, chosen uniformly at random. 

In this model, the scalar product $\langle \bx_j, \ba_r\rangle$ can be computed with an FFT algorithm. Furthermore, in order to evaluate the principal eigenvector for the spectral initialization, we use a power method which stops if either the number of iterations reaches the maximum value of $100000$ or the modulus of the scalar product between the estimate at the current iteration $T$ and at the iteration $T-10$ is larger than $1-10^{-7}$. 

The  GAMP algorithm with spectral initialization for the complex-valued setting is described in Appendix \ref{sec:complex_GAMP}.
Figure \ref{fig:venrec} shows a visual representation of the results. The improvement achieved by the GAMP algorithm over the spectral estimator is impressive, with GAMP achieving full recovery already at $\delta=2.4$. 
A numerical comparison of the performance of the two methods is given in Figure \ref{fig:cdp} in Appendix \ref{sec:complex_GAMP}.
We emphasize that the state evolution result of Theorem \ref{thm:GLM_spec} is only valid for Gaussian sensing matrices. Extending it to structured matrices such as coded diffraction patterns is an interesting direction for future work.


\section{Sketch of the Proof of Theorem \ref{thm:GLM_spec}} \label{sec:sketch}

We give an outline of the proof here, and provide the technical details in the appendices.

\vspace{1em}

\noindent \textbf{The artificial GAMP algorithm.} 
We construct an artificial GAMP algorithm, whose iterates are denoted by $\tbx^t, \tbu^t$, for $t \geq 0$. Starting from an initialization $(\tbx^0, \tbu^0)$,  for $t \geq 0$ we iteratively compute:
\begin{align}
\tbx^{t+1} &= \frac{1}{\sqrt{\delta}}\bA^\sT \th_t(\tbu^t; \by) - \tsc_t \tf_t(\tbx^t),  \label{eq:txk_update} \\
\tbu^{t+1} &= \frac{1}{\sqrt{\delta}}\bA \tf_{t+1}(\tbx^{t+1})-\tsb_{t+1} \th_t(\tbu^t; \by). \label{eq:tuk_update} 
\end{align}
For $t \geq 0$, the functions $\tf_t: \reals \to \reals$ and $\th_t: \reals \times \reals \to  \reals$ are Lipschitz, and will be specified below. The scalars $\tsc_t$ and $\tsb_{t+1}$ are defined as 
\begin{equation}
\tsc_t = \frac{1}{n}\sum_{i=1}^n \th_t'(\tu_i^t; y_i), \qquad \tsb_{t+1} =\frac{1}{n}\sum_{i=1}^d \tf_{t+1}'(\tx_i^{t+1}),
\label{eq:GAMP_onsager_tk}
\end{equation}
where $\th_t'$ denotes the derivative with respect to the first argument.   The iteration is initialized as follows.  Choose any $\alpha \in (0,1)$,  and a standard Gaussian vector $\bn \sim \normal(0, \bI_d)$ that is independent of $\bx$ and $\bA$. Then,
\begin{align}
\tbx^0 & =  \alpha \, \bx + \sqrt{1- \alpha^2} \, \bn, \qquad \tbu^0  = \frac{1}{\sqrt{\delta}} \bA \tf_0(\tbx^0). \label{eq:tbx0_tbu0}
\end{align}

The artificial GAMP is divided into two phases. In the first phase, which lasts up to iteration $T$, the functions $\tf_t, \th_t$ for $0 \leq t \leq (T-1)$, are chosen such that as $T\to\infty$, the iterate $\tbx^T$ approaches  the initialization $\bx^0$ of the true GAMP algorithm defined in \eqref{eq:x0_spec}. In the second phase, the functions $\tf_t, \th_t$ for $t \geq T$, are chosen to match those of the true GAMP.  The key observation is that a state evolution result for the artificial GAMP follows directly from the standard analysis of GAMP \cite{javanmard2013state} since the initialization $\tbx^0$ is independent of $\bA$. By showing that as $T \to \infty$, the iterates and the state evolution parameters of the artificial GAMP approach the corresponding quantities of the true GAMP, we prove that the state evolution result of Theorem \ref{thm:GLM_spec} holds.

We now specify the functions used in the artificial GAMP. For $0 \leq t \leq (T-1)$, we set
\begin{align}
\tf_t(x) = \frac{x}{\beta_t}, \qquad  \th_t(x; \, y) = \sqrt{\delta} \, x \, \frac{\cT_s(y)}{\lambda_\delta^*- \cT_s(y)},
\label{eq:tftg_leT}
\end{align}
where $\cT_s$ is the pre-processing function used for the spectral estimator, $\lambda_\delta^*$ is the unique solution of $\zeta_{\delta}(\lambda) = \phi(\lambda)$ for $\lambda > \tau$ (also given by \eqref{eq:corr_cond}),  and  $(\beta_t)_{t \geq 0}$ are constants coming from the state evolution recursion defined below. Furthermore, for $t\geq T$, we set
\begin{align}
\tf_t(x) = f_{t -T}(x), \qquad \th_{t}(x; \, y) = h_{t-T}(x; \, y).
\label{eq:tftg_geT}
\end{align}
With these choices of $\tf_t, \th_t$, the coefficients $\tsc_t$ and $\tsb_{t}$ in \eqref{eq:GAMP_onsager_tk} become:
\beq
\label{eq:onsager_art}
\begin{split}
&  \tsc_{t} = \frac{\sqrt{\delta}}{n} \sum_{i=1}^n \frac{\cT_s(y_i)}{\lambda_\delta^*- \cT_s(y_i)}, \,\, \tsb_t = \frac{1}{\delta\beta_t}, \,\,\,\, 0 \leq t \leq (T-1), \\
&  \tsc_t =  \frac{1}{n} \sum_{i=1}^n h_{t-T}'(\tu^t_i; \, y_i),  \,\, \tsb_t =  \frac{1}{n} \sum_{i=1}^d f_{t-T}'(\tx^t_i), \,\,\,\, t \geq T.
\end{split}
\eeq

Since the initialization $\tbx^0$ in \eqref{eq:tbx0_tbu0} is independent of $\bA$, the state evolution result of \cite{javanmard2013state} can be applied to the artificial GAMP. This result, formally stated in Proposition \ref{prop:tilSE} in Appendix \ref{app:artGAMP}, implies that for $t \geq 0$, the empirical distributions of $\tbx^t$ and $\tbu^t$  converge in $W_2$ distance to the laws of the random variables $\tX_t$ and $\tU_t$, respectively, with 
\begin{equation}
    \tX_t  \equiv \mu_{\tX,t} X + \sigma_{\tX,t} W_{\tX,t},    \quad 
        \tU_t  \equiv \mu_{\tU,t} G + \sigma_{\tU,t} W_{\tU,t}.  \label{eq:tXtUt_def}
\end{equation}
Here $W_{\tX,t}, W_{\tU,t}$ are standard normal and independent of $X$ and $G$, respectively. 
The state evolution recursion defining the parameters $(\mu_{\tX,t}, \sigma_{\tX,t}, \mu_{\tU,t}, \sigma_{\tU,t}, \beta_t)$ has the same form as \eqref{eq:SE_def}, except that we use the functions defined in \eqref{eq:tftg_leT} for $0 \leq t \leq (T-1)$, and the functions in \eqref{eq:tftg_geT} for $t \geq T$. The detailed expressions are given in Appendix \ref{app:artGAMP}.

\vspace{1em}

\noindent \textbf{Analysis of the first phase.} The first phase of the artificial GAMP is designed so that its output vectors after $T$ iterations $(\tbx^T, \tbu^T)$ are close to the initialization $(\bx^0, \bu^0)$ of the true GAMP algorithm given by  \eqref{eq:x0_spec}-\eqref{eq:u0_spec}. This part of the algorithm is similar to the GAMP used in \cite{mondelli2020optimal} to approximate the spectral estimator $\bxs$. In particular, the state evolution recursion of the first phase (given in \eqref{eq:defbetat1}) converges as $T \to \infty$ to the following fixed point: 
\beq
\lim_{T \to \infty}\mu_{\tX, T} =\frac{a}{\sqrt{\delta}}, \qquad  \lim_{T \to \infty}\sigma^2_{\tX, T}= \frac{1-a^2}{\delta},
\eeq
where $a$ is the limit (normalized) correlation between the spectral estimator $\bxs$ and the signal, see \eqref{eq:a2_explicit}. Furthermore, the GAMP iterate $\tbx^T$ approaches $\bxs$, i.e., 
\beq
\lim_{T \to \infty} \lim_{d \to \infty} \, \frac{\| \sqrt{d} \,  \bxs
     - \sqrt{\delta} \, \tbx^T \|^2}{d} =0 \ \text{ a.s.}
\eeq
These results are formally stated in Lemma \ref{lem:SE_FP_phase1} and \ref{lem:xsxT}, respectively, contained in Appendix \ref{app:firstphase}.

\vspace{1em}

\noindent \textbf{Analysis of the second phase.} The second phase of the artificial GAMP is designed so that its iterates $\tbx^{T+t}, \tbu^{T+t}$ are close to $\bx^t, \bu^t$,  respectively for $t \geq 0$,  and the corresponding state evolution parameters are also close. In particular, in order to prove Theorem \ref{thm:GLM_spec}, we first analyze a slightly modified version of the true GAMP algorithm in \eqref{eq:xt_update}-\eqref{eq:ut_update} where the `memory' coefficients $\sb_t$ and $\sc_t$ in \eqref{eq:GAMP_onsager} are replaced by deterministic values obtained from state evolution. The iterates of this modified GAMP, denoted by $\hbx^t, \hbu^t$, are as follows. Start with the  initialization
\begin{align}
&  \hbx^0=\bx^0= \sqrt{d} \, \, \frac{1}{\sqrt{\delta}}\bxs, \label{eq:x0_spec_mod} \\
& \hbu^0 = \frac{1}{\sqrt{\delta}} \bA f_{0}(\hbx^0) - \bsb_0\frac{\sqrt{\delta}}{\lambda_\delta^*} \,  \bZ_s  \bA \hbx^0, \label{eq:u0_spec_mod} 
\end{align}
 where $\bsb_0 = \frac{1}{\delta} \E\{ f_0'(X_0) \}$. Then, for $t \geq 0$:
\begin{align}
\hbx^{t+1} &= \frac{1}{\sqrt{\delta}}\bA^\sT h_t(\hbu^t; \by) - \bsc_t f_t(\hbx^t),  \label{eq:xt_update_mod} \\
\hbu^{t+1} &= \frac{1}{\sqrt{\delta}}\bA f_{t+1}(\hbx^{t+1})-\bsb_{t+1} h_t(\hbu^{t}; \by). \label{eq:ut_update_mod} 
\end{align}
Here, for $t \geq 0$,  the deterministic memory coefficients $\bsb_t$ and $\bsc_t$ are
\beq
\bsc_t = \E\{ h'_t(U_t; \, Y)\}, \qquad \bsb_t= \E\{ f_t'(X_t) \}/\delta,
\label{eq:bar_ons_terms}
\eeq
where  $X_t, U_t$ are defined in \eqref{eq:Xt_def}-\eqref{eq:Ut_def}.

Let us now summarize our approach. 
We have defined three different GAMP iterations: \emph{(i)} the \emph{true GAMP} with iterates
$(\bx^t, \bu^t)$ given by \eqref{eq:xt_update}-\eqref{eq:ut_update} and initialization $(\bx^0, \bu^0)$ given by \eqref{eq:x0_spec}-\eqref{eq:u0_spec}, \emph{(ii)} the \emph{modified GAMP} with iterates $(\hbx^t, \hbu^t)$ given by \eqref{eq:xt_update_mod}-\eqref{eq:ut_update_mod} and initialization $(\hbx^0, \hbu^0)$ given by \eqref{eq:x0_spec_mod}-\eqref{eq:u0_spec_mod}, and \emph{(iii)} the \emph{artificial GAMP} with iterates $(\tbx^t, \tbu^t)$ given by \eqref{eq:txk_update}-\eqref{eq:tuk_update} and initialization $(\tbx^0, \tbu^0)$ given by \eqref{eq:tbx0_tbu0}. We recall that the \emph{true GAMP} is the algorithm with spectral initialization that is actually implemented and whose performance we want to study. 
As the true GAMP is initialized with the spectral estimator $\bxs$ which depends on $\boldsymbol A$, its performance cannot be characterized using the existing theory. To solve this problem, we introduce the \emph{artificial GAMP} purely as a proof technique.
In fact, the initialization of the artificial GAMP assumes knowledge of the signal, which makes it impractical. Finally, the \emph{modified GAMP} is a slight modification of the true GAMP to simplify the proof.

Lemma \ref{lem:ArtTrueconn} in Appendix \ref{app:secondphase} proves that, for each $t\ge 0$, \emph{(i)} the iterates
$(\tbx^{t+T}, \tbu^{t+T})$ are close to $(\hbx^t, \hbu^t)$ for sufficiently large $T$, and \emph{(ii)}  the corresponding state evolution parameters are also close.  We then use this lemma to prove Theorem \ref{thm:GLM_spec} by showing that the iterates of the \emph{true GAMP} have the same asymptotic empirical distribution as those of the \emph{modified GAMP}. In particular, we show in in Appendix \ref{app:mainproof} that, almost surely,
\beq
\begin{split}
\lim_{d \to \infty} & \, \frac{1}{d} \sum_{i=1}^d \psi (x_{i},x^t_i) 
= \lim_{d \to \infty} \, \frac{1}{d} \sum_{i=1}^d \psi (x_{i},\hx^t_i) \\
&= \E \left\{ \psi( X, \, \mu_{X,t} X   + \sigma_{X,t} \, W) \right\}.
\end{split}
\eeq

\section{Discussion}

A major shortcoming in existing AMP theory for GLMs like  phase retrieval is the unrealistic assumption that  the initialization of the algorithm is correlated with the ground-truth signal and, at the same time, independent of the measurement matrix. This paper solves this problem by providing a rigorous analysis of AMP with a spectral initialization. Spectral initializations have been widely studied in recent years, and have two attractive features. First, for phase retrieval, they meet the information theoretic threshold for weak recovery \cite{mondelli2017fundamental}. This means that, when the spectral initialization fails, no other method can work. Second, for a large class of GLMs, if the spectral method is unsuccessful, then AMP has an attractive fixed point at $0$, see Theorem 5 in \cite{mondelli2017fundamental}. This is a strong indication that, when the spectral initialization fails, the problem is computationally hard. An interesting future direction is to analyze the fixed points of AMP with spectral initialization, and compare  with those of other algorithms that can be initialized with a spectral  estimator, e.g., gradient descent.

Our analysis is based on an artificial AMP that first closely approximates the spectral estimator and then the true AMP algorithm. This technical tool is  versatile and could be used beyond GLMs with Gaussian sensing matrices. Examples include more general measurement models \cite{fan2020approximate,emami2020generalization}, other message passing algorithms, e.g., Vector AMP \cite{schniter2016vector}, or the design of an artificial AMP that leads to a different estimator. We also highlight that the AMP analyzed here is rather general, and it includes as special cases both the Bayes-optimal AMP for GLMs and AMPs designed to optimize objective functions tailored to the signal prior.

\section*{Acknowledgements}

The authors would like to thank Andrea Montanari for  helpful discussions. M. Mondelli was partially supported by the 2019 Lopez-Loreta Prize. R. Venkataramanan was partially supported by the Alan Turing Institute under the EPSRC grant EP/N510129/1.

{\small{
\bibliographystyle{amsalpha}
\bibliography{all-bibliography}
\addcontentsline{toc}{section}{References}
}}

\appendix

\section{Proof of Proposition \ref{prop:opt_gt}} \label{app:proof_opt_gt}

From assumption \textbf{(B2)} on p. \pageref{assump:h_zv}, we recall that $h_t(u; \, y) = h_t(u; \, q(g, v))$. We write 
$\partial_g h_t(u; \, q(g,v))$ for the partial derivative with respect to $g$. We will show that $\mu_{X,t+1}$ in \eqref{eq:SE_def} can be written as: 
\begin{align}
\mu_{X,t+1} & = \sqrt{\delta}\E\{ \partial_g h_t(U_t; \, q(G,V)) \},
\label{eq:muXt1_alt1} \\
& = \sqrt{\delta} \E \left\{ h_t(U_t; Y) \left(  \frac{\E\{ G | U_t, Y \} - \E\{ G | U_t \} }{{\sf Var}\{G  | U_t \}} \right) \right\}.
\label{eq:muXt1_alt2}
\end{align}
From \eqref{eq:muXt1_alt2}, we have that 
\beq
\frac{\mu_{X,t+1}}{\sigma_{X,t+1}} = \frac{\sqrt{\delta}}{\sqrt{\E\{ h_t(U_t; Y)^2 \}}}
\E \left\{ h_t(U_t; Y) \left(  \frac{\E\{ G | U_t, Y \} - \E\{ G | U_t \} }{{\sf Var}\{G  | U_t \}} \right) \right\}.
\eeq
The absolute value of the RHS is maximized when $h_t = c \,  h_t^*$, for $c \neq  0$ and $h_t^*$ is given in \eqref{eq:opt_gt}. To obtain the alternative expression in \eqref{eq:opt_gt_alt2} 
from \eqref{eq:opt_gt}  we recall that $U_t$ is Gaussian with zero mean and variance $(\mu_{U,t}^2 + \sigma_{U,t}^2)$. Furthermore,  the conditional distribution of $G$ given $U_t=u$ is Gaussian with $\E\{ G \mid U_t =u\} = \rho_t u$ and ${\sf Var}(G \mid U_t=u) = (1 - \rho_t \mu_{U,t})$. Therefore, with $W \sim \normal(0,1)$ we have
\begin{align}
\E\{ G \mid  U_t =u, \, Y =y  \} & = 
\frac{\E_W\{ (\rho_t u + \sqrt{1-\rho_t\, \mu_{U,t}}\, W)\, p_{Y|G}( y \, | \, \rho_t u + \sqrt{1-\rho_t\, \mu_{U,t}}\, W) \}}{\E_W\{ p_{Y|G}( y \, | \, \rho_t u + \sqrt{1-\rho_t \, \mu_{U,t}}\, W)\}} \nonumber \\
& = \rho_t u \, +  \, \sqrt{1-\rho_t\, \mu_{U,t}} \, \frac{\E\{ W p_{Y|G}( y \, | \, \rho_t u + \sqrt{1-\rho_t\, \mu_{U,t}}\, W) \}}{\E_W\{ p_{Y|G}( y \, | \, \rho_t u + \sqrt{1-\rho_t \, \mu_{U,t}}\, W)\}}. \label{eq:alt_exp1}
%
\end{align}
Substituting \eqref{eq:alt_exp1}  in \eqref{eq:opt_gt}  yields \eqref{eq:opt_gt_alt2}.

It remains to show \eqref{eq:muXt1_alt2}, which we do by first showing \eqref{eq:muXt1_alt1}. Define  $e_t: \reals^3 \to \reals$ by
\beq
e_t(g, w, v) = h_t(\mu_{U,t}g + \sigma_{U,t} w; \, q(g, v)).
\eeq
Then, using the chain rule, the partial derivative of $e_t(g, w, v)$ with respect to $g$ is
\beq
\partial_g e_t(g, w, v) = \mu_{U,t} h_t'(\mu_{U,t}g + \sigma_{U,t} w; \, q(g, v))
\, +  \, \partial_g h_t(u; \, q(g,v)).
\label{eq:et_partg}
\eeq
The parameter  $\mu_{X,t+1}$ in  \eqref{eq:SE_def} can be written as
\begin{align}
\mu_{X, t+1} & = \sqrt{\delta}\left[ \E\{G \, e_t(G, W_{U,t}, V) \} - \mu_{U,t} \E\{ h_t'(\mu_{U,t}G+ \sigma_{U,t}W_{U,t}; \, Y)\} \right] \nonumber \\
& \stackrel{{\rm (i)}}{=} 
\sqrt{\delta}\left[\E \big\{ \partial_g e_t(G, W_{U,t}, V) \big\} - \mu_{U,t} \E\{ h_t'(\mu_{U,t}G+ \sigma_{U,t}W_{U,t}; \, Y)\} \right] \nonumber \\
& = \sqrt{\delta} \, \E\{  \partial_g h_t(U_t; \, q(G,V)) \},
\end{align}
where the last equality is due to \eqref{eq:et_partg}, and {\rm (i)} holds due to Stein's lemma. Finally, we obtain \eqref{eq:muXt1_alt2} from \eqref{eq:muXt1_alt1} as follows:
\beq
\begin{split}
 \E\{ \partial_g h_t( U_t; \, q(G, V) )\} & = \E \left\{ \E_{G|U_t} \big[  \partial_g  h_t( U_t; \, q(G, V) ) \mid U_t  \big]   \right\} \\
& \stackrel{\rm (ii)}{=} \E \left\{ \E_{G|U_t}\big[  h_t( U_t; \, q(G, V)) \cdot (G - \E\{G|U_t \})/{\sf Var}\{G  | U_t \} \mid U_t \big] \right\} \\
& = \E \left\{ \E_{G|U_t, Y} \big[   h_t( U_t; \, Y ) \cdot (G - \E\{G|U_t \})/{\sf Var}\{G  | U_t \} \mid U_t, Y \big] \right \} \\
& = \E \left\{ h_t(U_t; Y) \cdot \left(  (\E\{ G | U_t, Y \} - \E\{ G | U_t \})/{\sf Var}\{G  | U_t \} \right) \right\}.
\end{split}
\eeq
Here step {\rm (ii)} holds due to Stein's lemma. This completes the proof of the proposition.  \qed

\section{Proof of the Main Result}

\subsection{The Artificial GAMP Algorithm}\label{app:artGAMP}

The state evolution parameters for the artificial GAMP are recursively defined as follows.
Recall from \eqref{eq:tXtUt_def} that $\tX_t  = \mu_{\tX,t} X + \sigma_{\tX,t} W_{\tX,t}$ and $\tU_t  \equiv \mu_{\tU,t} G + \sigma_{\tU,t} W_{\tU,t}$. 
Using \eqref{eq:tbx0_tbu0}, the state evolution initialization is
\begin{align}
\mu_{\tX, 0} = \alpha, \qquad   \sigma^2_{\tX, 0} = 1- \alpha^2, \qquad  \beta_0 =\sqrt{\mu_{\tX, 0}^2 + \sigma^2_{\tX, 0} } =1.
\label{eq:SE_tinit}
\end{align}
For $0 \leq t \leq (T-1)$, the state evolution parameters are iteratively computed by using the functions defined in \eqref{eq:tftg_leT} in \eqref{eq:SE_def}:
\begin{align}
    & \mu_{\tU,t} = \frac{\mu_{\tX,t}}{\sqrt{\delta} \beta_t}, \qquad \sigma_{\tU, t}^2 = \frac{\sigma_{\tX,t}^2}{ \delta \, \beta^2_t}, \nonumber \\
    &\mu_{\tX,t+1} =   \frac{\mu_{\tX,t}}{\sqrt{\delta}\beta_t} ,\qquad \sigma_{\tX,t+1}^2 = \frac{1}{\beta_t^2} \, \E\left\{ \frac{Z_s^2 (G^2\mu_{\tX,t}^2+\sigma_{\tX,t}^2)}{(\lambda_\delta^*-Z_s)^2}\right\} ,     \nonumber    \\
    & \beta_{t+1} = \sqrt{\mu_{\tX,{t+1}}^2 + \sigma_{\tX,{t+1}}^2}. \label{eq:defbetat1}
\end{align}
Here we recall that $G \sim \normal(0,1)$, $Y \sim p_{Y|G}(\cdot \mid G)$, $Z_s=\cT_s(Y)$, and the equality in \eqref{eq:corr_cond} which is used to obtain the expression for $\mu_{\tX,t+1}$.
For $t \geq T$,  the state evolution parameters are:
\begin{align}
    \mu_{\tU,t} & = \frac{1}{\sqrt{\delta}} \E\{ X f_{t-T}(\tX_t) \}, \nonumber \\
     \sigma_{\tU,t}^2 & = 
     \frac{1}{\delta} \E\big\{ f_{t-T}(\tX_t)^2 \big\} - \mu_{\tU,t}^2, \nonumber \\
     \mu_{\tX,t+1}& =\sqrt{\delta} \E\{ G h_{t-T}(\tU_t; Y) \} \, - \, \E\{ h_{t-T}'(\tU_t;Y) \} \E\{ X f_{t-T}(\tX_t) \}, \nonumber \\
     \sigma_{\tX,t+1}^2 & = \E\{ h_{t-T}(\tU_t; Y)^2 \}. \label{eq:SE_def_til}
\end{align}

\begin{proposition}[State evolution for artificial GAMP]
\label{prop:tilSE}
Consider the setting of Theorem \ref{thm:GLM_spec}, the artificial GAMP iteration described in \eqref{eq:txk_update}-\eqref{eq:onsager_art}, and the corresponding state evolution parameters defined in \eqref{eq:SE_tinit}-\eqref{eq:SE_def_til}.

For any \PL(2) function $\psi: 
\reals^{2}  \to \reals$, the following holds almost surely for $t \geq 1$:
\begin{align}
& \lim_{d \to \infty} \frac{1}{d} \sum_{i=1}^d \psi(x_i, \tx^{t}_i) = \E\{ \psi(X, \, \tX_t) \}, \label{eq:psiXt}\\
& \lim_{n \to \infty}\frac{1}{n} \sum_{i=1}^n \psi(y_i, \tu^t_i) = \E\{ \psi( Y,  \, \tU_t) \}. \label{eq:psiGt}
\end{align}
Here $X\sim P_X$ and  $Y \sim P_{Y|G}$, with $G \sim \normal(0,1)$.  The random variables $\tX_t, \tU_t$ are defined in \eqref{eq:tXtUt_def}. 
\end{proposition}

The proposition follows directly from the state evolution result of \cite{javanmard2013state} since the initialization $\tbx^0$ of the artificial GAMP is independent of $\bA$.

\subsection{Analysis of the First Phase}\label{app:firstphase}

\begin{lemma}[Fixed point of state evolution for first phase] Consider the setting of Theorem \ref{thm:GLM_spec}. 
Then, the state evolution recursion for the first phase, given by \eqref{eq:SE_tinit}-\eqref{eq:defbetat1}, converges as $T \to \infty$ to the following fixed point: 
\beq
\mu_{\tX} \triangleq  \lim_{T \to \infty}\mu_{\tX, T} =\frac{a}{\sqrt{\delta}}, \qquad  \sigma_{\tX}^2 \triangleq\lim_{T \to \infty}\sigma^2_{\tX, T}= \frac{1-a^2}{\delta},
\label{eq:musig_a}
\eeq
where $a$ is defined in \eqref{eq:a2_explicit}.
\label{lem:SE_FP_phase1}
\end{lemma}

\begin{proof}
Recall that $\lambda_\delta^*$ denotes the unique solution of $\zeta_{\delta}(\lambda) = \phi(\lambda)$ for $\lambda > \tau$ (also given by \eqref{eq:corr_cond}), and define $Z=Z_s/(\lambda_\delta^*-Z_s)$, where $Z_s=\cT_s(Y)$. Note that
\beq \label{eq:simplify1}
\E\{ Z(G^2-1)\} = \E \left\{ \frac{Z_s(G^2-1)}{\lambda_\delta^*- Z_s} \right\} = \frac{1}{\delta},
\eeq
where the second equality follows from the equality in \eqref{eq:corr_cond}. Moreover, the inequality in \eqref{eq:corr_cond} implies that 
\beq\label{eq:simplify2}
\frac{\E\{ Z^2\}}{ (\E\{ Z (G^2-1) \} )^2} = \delta^2 \E\left\{ \frac{Z_s^2}{(\lambda_\delta^* - Z_s)^2}\right\} < \delta.
\eeq
Thus, by recalling that the state evolution initialization $\mu_{\tX, 0} = \alpha$ is strictly positive, the result follows from Lemma 5.2 in \cite{mondelli2020optimal}.
\end{proof}

\begin{lemma}[Convergence to spectral estimator]
Consider the setting of Theorem \ref{thm:GLM_spec}, and consider the first phase of the artificial GAMP iteration, given by \eqref{eq:txk_update}-\eqref{eq:tuk_update} with $\tf_t$ and $\th_t$ defined in \eqref{eq:tftg_leT}. Then, 
\beq
\lim_{T \to \infty} \lim_{d \to \infty} \, \frac{\| \sqrt{d} \,  \bxs
     - \sqrt{\delta} \, \tbx^T \|^2}{d} =0 \ \text{ a.s.}
     \label{eq:xstxT_diff}
\eeq
Furthermore, for any \PL(2) function $\psi: \reals \times \reals \to \reals$, almost surely we have:
\begin{align}
& \lim_{d \to \infty} \,   \frac{1}{d} \sum_{i=1}^d \psi(x_i,  \sqrt{d} \, \hat{x}^{\rm s}_{i}) = \lim_{T \to \infty}  \lim_{d \to \infty}  \frac{1}{d} \sum_{i=1}^d \psi(x_i, \sqrt{\delta} \, \tx_i^{T}) =  
\E\{ \psi(X  ,\,  \sqrt{\delta}\, (\mu_{\tX} X + \sigma_{\tX} W )) \}.
 \label{eq:psiX1_sp}
\end{align}
Here $X \sim P_X$ and $W \sim \normal(0,1)$  are independent.
\label{lem:xsxT}
\end{lemma}

\begin{proof}
As in the proof of the previous result, let $Z=Z_s/(\lambda_\delta^*-Z_s)$ and note that \eqref{eq:simplify1}-\eqref{eq:simplify2} hold. Also define
\beq
Z' \triangleq \frac{Z}{Z+\delta \E\{ Z (G^2-1) \}}=\frac{Z}{Z+1}=\frac{Z}{\lambda_\delta^*}.
\eeq
Then, the assumptions of Lemma 5.4 in \cite{mondelli2020optimal} are satisfied, with the only difference of the initialization of the GAMP iteration (cf. \eqref{eq:tbx0_tbu0} in this paper and (5.4) in \cite{mondelli2020optimal}). However, it is straightforward to verify that the difference in the initialization does not affect the proof of Lemma 5.4 in \cite{mondelli2020optimal}. Thus, \eqref{eq:xstxT_diff} follows from (5.87) of \cite{mondelli2020optimal}, and \eqref{eq:psiX1_sp} follows by taking $k=2$ in (5.31) of \cite{mondelli2020optimal}.
\end{proof}

We will also need the following result on the convergence of the GAMP iterates.
\begin{lemma}[Convergence of GAMP iterates]
Consider the first phase of the artificial GAMP iteration, given by \eqref{eq:txk_update}-\eqref{eq:tuk_update} with $\tf_t$ and $\th_t$ defined in \eqref{eq:tftg_leT}. Then, the following limits hold almost surely:
\beq
\lim_{T \to \infty} \lim_{n \to \infty} \, \frac{1}{n} \| \tbu^{T-1} - \tbu^{T-2} \|_2^2 =0, \qquad \lim_{T \to \infty} \lim_{d \to \infty} \frac{1}{d} \| \tbx^{T} - \tbx^{T-1} \|_2^2 =0.
\eeq
\label{lem:diffs}
\end{lemma}

Though the initialization of the GAMP in \cite{mondelli2020optimal} is different from \eqref{eq:tbx0_tbu0}, the proof of Lemma \ref{lem:diffs} is the same as that of Lemma 5.3 in \cite{mondelli2020optimal} since it only relies on $\mu_{\tX,0} = \alpha$ being strictly non-zero.

\subsection{Analysis of the Second Phase}\label{app:secondphase}

\begin{lemma}
Assume the setting of Theorem \ref{thm:GLM_spec}. Consider the artificial GAMP algorithm \eqref{eq:txk_update}-\eqref{eq:tuk_update} with the related state evolution recursion \eqref{eq:defbetat1}-\eqref{eq:SE_def_til}, and  the modified version of the true GAMP algorithm \eqref{eq:xt_update_mod}-\eqref{eq:ut_update_mod}. Fix any $\eps >0$. Then, for $t \geq 0$ such that $\sigma_{X,k}^2 >0$ for $0 \leq k \leq t$,  the following statements hold:
\begin{enumerate}
\item  
\begin{align}
& \lim_{T \to \infty} \abs{\mu_{\tU,t+T} - \mu_{U,t} } =0,  \qquad \qquad \lim_{T \to \infty} \abs{\sigma^2_{\tU,t+T} - \sigma^2_{U,t} } =0,
\label{eq:SE_gapU}\\
& \lim_{T \to \infty} \abs{\mu_{\tX,T+ t+1} - \mu_{X,t+1} }=0, \qquad \lim_{T \to \infty} \abs{\sigma^2_{\tX,T+t+1} - \sigma^2_{X,t+1} } =0. \label{eq:SE_gapX} 
\end{align}

\item  Let $\psi: \reals \times \reals \to \reals$ be a \PL(2) function. Then,  almost surely,
\begin{align}
&  \lim_{T \to  \infty} \lim_{n \to \infty} \abs{  \frac{1}{n} \sum_{i=1}^n \psi(y_i, \tu^{T+t}_i)  - \frac{1}{n} \sum_{i=1}^n \psi(y_i, \hu^t_i) } =0,  \label{eq:uhat_til}  \\
 & \lim_{T \to  \infty} \lim_{d \to \infty} \abs{ \frac{1}{d} \sum_{i=1}^d \psi(x_i, \tx^{T+t+1}_i)  - \frac{1}{d} \sum_{i=1}^d \psi(x_i, \hx^{t+1}_i) } =0. \label{eq:xhat_til}
\end{align}
\end{enumerate}
The limits in \eqref{eq:SE_gapX} and \eqref{eq:xhat_til} also hold for $t+1=0$.
\label{lem:ArtTrueconn}
\end{lemma}

\begin{proof}

We will use $\kappa_t, \kappa'_t, c_t, \gamma_t$ to denote generic positive constants which depend on $t$, but not on $n, d$, or $\eps$. The values of these constants may change throughout the proof. 


\vspace{1em}

\noindent  \textbf{Proof of \eqref{eq:SE_gapU} and  \eqref{eq:SE_gapX}.}
We prove the result by induction, starting from the base case $\abs{\mu_{\tX,T} - \mu_{X,0} }$,  $| \sigma^2_{\tX,T} - \sigma^2_{X,0} |$. From Lemma \ref{lem:SE_FP_phase1}, we  have
\beq
\lim_{T \to \infty} \mu_{\tX, T} = \mu_{\tX} = \frac{a}{\sqrt{\delta}}, \qquad 
\lim_{T \to \infty} \sigma^2_{\tX, T} = \sigma^2_{\tX} = \frac{1-a^2}{\delta}.
\label{eq:TFP_SE}
\eeq
Recalling from \eqref{eq:SE_init} that $\mu_{X,0} =\frac{a}{\sqrt{\delta}}$, $\sigma_{X,0}^2 = \frac{1-a^2}{\delta}$, \eqref{eq:TFP_SE} implies that 
\beq
\lim_{T \to \infty} \, \abs{\mu_{\tX,T} - \mu_{X,0} }=0, \qquad 
\lim_{T \to \infty} \, \abs{\sigma^2_{\tX,T} - \sigma^2_{X,0} }=0.
\label{eq:base_step}
\eeq

Assume towards induction that \eqref{eq:SE_gapX} holds with $(t+1)$ replaced by $t$, and that $\sigma_{X,k}^2 >0$ for $0 \leq k \leq t$. We will show that \eqref{eq:SE_gapU} holds, and then that \eqref{eq:SE_gapX} holds.

For brevity, we write $\Delta_{\mu,t}, \Delta_{\sigma,t}$ for $(\mu_{X,t} - \mu_{\tX, t+T})$ and $(\sigma_{X,t} - \sigma_{\tX, t+T})$, respectively. By the induction hypothesis, given any $\eps >0$, for $T$ sufficiently large we have
\beq
\abs{\Delta_{\mu,t}} < \kappa_t \eps, \qquad \abs{\Delta_{\sigma,t}}< \frac{\kappa_t}{\sigma_{X,t} + \sigma_{\tX, t+T}} \eps = \kappa'_t \eps.
\label{eq:Delmusig_bnds}
\eeq
Since $\sigma_{X,t}$ is  strictly positive, $\kappa_t'$ is finite and bounded above.

From \eqref{eq:SE_def} we have
\begin{align}
\mu_{U,t} & =  \frac{1}{\sqrt{\delta}}\E\{ X f_t( \mu_{X,t}X + \sigma_{X,t} W_{X,t} ) \} = \frac{1}{\sqrt{\delta}}\E\{ X f_t( \mu_{\tX,T+t}X + \sigma_{\tX,T+t} W_{X,t} 
\, + \,  \Delta_{\mu, t}X + \Delta_{\sigma,t} W_{X,t} ) .
\label{eq:mu_Del1}
\end{align}
Recalling that $f_t$ is Lipschitz and letting $L_t$ denote its Lipschitz constant, we have
\beq
\begin{split}
& \abs{f_t( \mu_{\tX,T+t}X + \sigma_{\tX,T+t} W_{X,t} 
\, + \,  \Delta_{\mu, t}X + \Delta_{\sigma,t} W_{X,t} ) 
\, - \, f_t( \mu_{\tX,T+t}X + \sigma_{\tX,T+t} W_{X,t})} \\
&\hspace{27em}\leq L_t \abs{\Delta_{\mu, t}X + \Delta_{\sigma,t} W_{X,t}}.
\end{split}
\label{eq:ft_Lip}
\eeq
Using \eqref{eq:ft_Lip} in \eqref{eq:mu_Del1}, we obtain
\beq
\begin{split}
\sqrt{\delta}\mu_{U,t}\geq & \, \E\{ X f_t( \mu_{\tX,T+t}X + \sigma_{\tX,T+t} W_{X,t}) \} \, - \, 
L_t \E\{\abs{X} \abs{\Delta_{\mu, t}X + \Delta_{\sigma,t} W_{X,t}} \}, \\
\sqrt{\delta}\mu_{U,t}  
 \leq & \,\E\{ X f_t( \mu_{\tX,T+t}X + \sigma_{\tX,T+t} W_{X,t}) \} \, + \, 
L_t \E\{\abs{X} \abs{\Delta_{\mu, t}X + \Delta_{\sigma,t} W_{X,t}} \}.
\end{split}
\label{eq:muUt_bnd1}
\eeq
Since $W_{X,t} \stackrel{\de}{=} W_{\tX,t+T}$ and independent of $X$, we have that $\E\{ X f_t( \mu_{\tX,T+t}X + \sigma_{\tX,T+t} W_{X,t}) \} = \sqrt{\delta} \mu_{\tU,t+T}$. Therefore, \eqref{eq:muUt_bnd1} implies
 \beq
 \sqrt{\delta}\abs{\mu_{U,t} - \mu_{\tU,t+T}} \, \leq \,  L_t(\Delta_{\mu,t} \, +  \, \Delta_{\sigma,t} \E\{ \abs{W_{X,t}} \} ),
\eeq
where we have used $\E\{ \abs{X}^2 \} <  \sqrt{\E\{ X^2 \}} =1$. Noting that  $\E\{ \abs{W_{X,t}} \} = \sqrt{2/\pi}$, from \eqref{eq:Delmusig_bnds} it follows that for sufficiently large $T$:
\beq
\abs{\mu_{U,t} - \mu_{\tU,t+T}} \leq  \frac{L_t}{\sqrt{\delta}} ( \kappa_{t}   +  \kappa'_{t} \sqrt{2/\pi} ) \,  \eps \,  < \, \gamma_t \, \eps.\label{eq:bdmunew}
\eeq

Next consider $\sigma_{U,t}^2$. From \eqref{eq:SE_def}, we have
\beq
\sigma_{U,t}^2 = \frac{1}{\delta} \E\{ f_t(\mu_{X,t} X + \sigma_{X,t} W_{X,t})^2 \} - \mu_{U,t}^2.
\label{eq:sigUtexp}
\eeq
Furthermore, as $W_{X,t} \stackrel{\de}{=} W_{\tX,t+T}$ and independent of $X$, we also have that
\begin{equation}
    \sigma_{\tU,t+T}^2 = \frac{1}{\delta} \E\{ f_t(\mu_{\tX,t+T} X + \sigma_{\tX,t+T} W_{X,t})^2 \} - \mu_{\tU,t+T}^2.\label{eq:sigUtexp2}
\end{equation}
Using the reverse triangle inequality, we have
\beq
\begin{split}
& \abs{ f_t( \mu_{\tX,T+t}X + \sigma_{\tX,T+t} W_{X,t} 
\, + \,  \Delta_{\mu, t}X + \Delta_{\sigma,t} W_{X,t})} \\
&\hspace{3em}\geq \lvert f_t( \mu_{\tX,T+t}X + \sigma_{\tX,T+t} W_{X,t}) \rvert\\  &\hspace{6em}-\abs{f_t( \mu_{\tX,T+t}X + \sigma_{\tX,T+t} W_{X,t} 
\, + \,  \Delta_{\mu, t}X + \Delta_{\sigma,t} W_{X,t} ) 
\, - \, f_t( \mu_{\tX,T+t}X + \sigma_{\tX,T+t} W_{X,t})} \\
& \hspace{3em} \geq \lvert f_t( \mu_{\tX,T+t}X + \sigma_{\tX,T+t} W_{X,t}) \rvert \, - \, 
L_t \abs{\Delta_{\mu, t}X + \Delta_{\sigma,t} W_{X,t}},
\end{split}
\label{eq:absft_LB}
\eeq
where the last inequality follows from \eqref{eq:ft_Lip}. Similarly, 
\beq
\begin{split}
& \abs{ f_t( \mu_{\tX,T+t}X + \sigma_{\tX,T+t} W_{X,t} 
\, + \,  \Delta_{\mu, t}X + \Delta_{\sigma,t} W_{X,t})} \\
& \hspace{3em} \leq \lvert f_t( \mu_{\tX,T+t}X + \sigma_{\tX,T+t} W_{X,t}) \rvert \, + \, 
L_t \abs{\Delta_{\mu, t}X + \Delta_{\sigma,t} W_{X,t}}.
\end{split}
\label{eq:absft_UB}
\eeq
Using \eqref{eq:absft_LB}, we obtain the bound 
\beq
\begin{split}
& \E\{ f_t(\mu_{X,t} X + \sigma_{X,t} W_{X,t})^2 \}    \geq \, \E\{ f_t( \mu_{\tX,T+t}X + \sigma_{\tX,T+t} W_{X,t}) ^2 \} 
- L_t^2 \E\{ \abs{\Delta_{\mu, t}X + \Delta_{\sigma,t} W_{X,t}}^2\} \\
& \qquad \qquad 
- 2 L_t \sqrt{ \E\{f_t( \mu_{X,t}X + \sigma_{X,t} W_{X,t}) ^2 \} 
\cdot \E\{ \abs{\Delta_{\mu, t}X + \Delta_{\sigma,t} W_{X,t}}^2\} }.
\end{split}
\label{eq:absft_LB2}
\eeq
Similarly, using \eqref{eq:absft_UB} we get 
\beq
\begin{split}
& \E\{ f_t(\mu_{X,t} X + \sigma_{X,t} W_{X,t})^2 \}    \leq \, \E\{ f_t( \mu_{\tX,T+t}X + \sigma_{\tX,T+t} W_{X,t}) ^2 \} 
+ L_t^2 \E\{ \abs{\Delta_{\mu, t}X + \Delta_{\sigma,t} W_{X,t}}^2\} \\
& \qquad \qquad 
+ 2 L_t \sqrt{ \E\{f_t( \mu_{\tX,T+t}X + \sigma_{\tX,T+t} W_{X,t}) ^2 \} 
\cdot \E\{ \abs{\Delta_{\mu, t}X + \Delta_{\sigma,t} W_{X,t}}^2\} }.
\end{split}
\label{eq:absft_UB2}
\eeq
Furthermore,
\[ \E\{ \abs{\Delta_{\mu, t}X + \Delta_{\sigma,t} W_{X,t}}^2\}  \leq 
2\abs{\Delta_{\mu, t}}^2 \E\{ X^2 \} + 2\abs{\Delta_{\sigma, t}}^2 \E\{W_{X,t}\}^2 
=2 (\abs{\Delta_{\mu, t}}^2 + \abs{\Delta_{\sigma, t}}^2). \]
From \eqref{eq:SE_def} and \eqref{eq:SE_def_til}, we note that 
\beq
\begin{split}
& \E\{ f_t( \mu_{X,t}X + \sigma_{X,t} W_{X,t}) \} ^2 = \delta (\mu_{U,t}^2 + \sigma_{U,t}^2),  \\
& \E\{ f_t( \mu_{\tX, T+t}X + \sigma_{\tX,T+t} W_{\tX,T+t})\}^2 = \delta (\mu_{\tU,T+t}^2 + \sigma_{\tU, T+t}^2).
\end{split}
\label{eq:Eft2}
\eeq
Therefore, Eqs. \eqref{eq:absft_LB2} and \eqref{eq:absft_UB2} imply that
\beq
\begin{split}
& \abs{\E\{ f_t(\mu_{X,t} X + \sigma_{X,t} W_{X,t})^2 \}  - 
\E\{ f_t( \mu_{\tX,T+t}X + \sigma_{\tX,T+t} W_{X,t}) ^2 \}} \\ 
&  \leq 2L_t^2 (\abs{\Delta_{\mu, t}}^2 + \abs{\Delta_{\sigma, t}}^2) 
+ 2L_t \sqrt{ 2\delta (\mu_{U, t}^2 + \sigma_{U, t}^2 +\mu_{\tU, T+t}^2 + \sigma_{\tU, T+t}^2) (\abs{\Delta_{\mu, t}}^2 + \abs{\Delta_{\sigma, t}}^2)}.
\end{split}
\label{eq:ft2_bnd}
\eeq
Using this in \eqref{eq:sigUtexp}-\eqref{eq:sigUtexp2}, we have 
\begin{equation}
    \begin{split}
        & \abs{\sigma^2_{U,t} - \sigma^2_{\tU,t+T}} \leq \abs{\mu_{\tU, T+t} - \mu_{U, t}}\cdot \abs{\mu_{\tU, T+t} + \mu_{U, t}} \\
        & \qquad \quad  + \bigg(  \frac{2}{\delta} L_t^2 (\abs{\Delta_{\mu, t}}^2 + \abs{\Delta_{\sigma, t}}^2) 
        + \frac{2}{\sqrt{\delta}} L_t \sqrt{2 (\mu_{U, t}^2 + \sigma_{U, t}^2 +\mu_{\tU, T+t}^2 + \sigma_{\tU, T+t}^2) (\abs{\Delta_{\mu, t}}^2 + \abs{\Delta_{\sigma, t}}^2)} \, \bigg). \\
    \end{split}
    \label{eq:sig2_diff}
\end{equation}
From \eqref{eq:Delmusig_bnds}, we obtain 
\beq
\abs{\Delta_{\mu, t}}^2 + \abs{\Delta_{\sigma, t}}^2 < \big(\kappa_t^2 + (\kappa'_t)^2\big) \eps^2.
\label{eq:sumDelsig}
\eeq
Furthermore, as $f_t$ is Lipschitz, from \eqref{eq:Eft2} and the induction hypothesis we have
\begin{equation}
    |\mu_{\tU, T+t}|+\abs{\mu_{U, t}}+\sigma_{U, t}+\sigma_{\tU, T+t}\le c_t, \label{eq:bdfinite}
\end{equation}
for some constant $c_t$. Using \eqref{eq:bdmunew}, \eqref{eq:sumDelsig} and \eqref{eq:bdfinite} in \eqref{eq:sig2_diff}, we conclude that for sufficiently large $T$:
\beq
\abs{\sigma^2_{U,t} - \sigma^2_{\tU,T+t}} < \gamma_t \eps.
\eeq

Next, we show that if \eqref{eq:SE_gapU} holds for some $t \geq 0$ and $\sigma_{X,k}^2 >0$ for $k \leq t$, then :
\begin{align}
 & \lim_{T \to \infty}\abs{\mu_{\tX,T+t+1} - \mu_{X,t+1} } =0, \qquad \lim_{T \to \infty}\abs{\sigma^2_{\tX,T+t+1} - \sigma^2_{X,t+1} } =0. \label{eq:SE_gapXt1}
\end{align}
 We denote the Lipschitz constant of $h_t$ by $\olL_t$, and write $\olDel_{\mu, t}$, $\olDel_{\sigma, t}$ for $(\mu_{U,t} - \mu_{\tU, t+T})$ and $(\sigma_{U,t} - \sigma_{\tU, t+T})$, respectively.  Using this notation, we have 
\begin{equation}
    \begin{split}
    & \abs{h_t(\mu_{\tU, T+t} G + \sigma_{\tU ,T+t} W_{U,t} + 
    \olDel_{\mu,t} G + \olDel_{\sigma,t} W_{U,t}; \, Y) - h_t(\mu_{\tU, T+t} G + \sigma_{\tU ,T+t} W_{U,t}; \, Y) }  \\
    & \leq \olL_t \abs{ \olDel_{\mu,t} G + \olDel_{\sigma,t} W_{U,t} }.
    \label{eq:ht_lip}
    \end{split}
\end{equation}
The induction hypothesis \eqref{eq:SE_gapU} implies that for sufficiently large $T$:
\beq
\abs{\olDel_{\mu,t}} < \gamma_t \eps, \qquad \abs{\olDel_{\sigma,t}}< \frac{\gamma_t}{\sigma_{U,t} + \sigma_{\tU, t+T}} \eps = \gamma_t \eps.
\label{eq:olDelmusig_bnds}
\eeq
We note that $\sigma_{U,t} >0$ since $\sigma_{X,t} >0$. Indeed, from the discussion leading to \eqref{eq:muU_Bopt}, for a fixed $\mu_{X,t}, \sigma_{X,t}$ the smallest possible ratio $\sigma_{U,t}^2/\mu_{U,t}^2$ is achieved by the Bayes-optimal choice $f_t = c f_t^*$, where $f_t^*(X_t) = E\{X | X_t\}$.  Furthermore, from \eqref{eq:muU_Bopt}, in order for $\sigma_{U,t}=0$, we need $\E\{ \E\{X |X_t\}^2 \}=1$. From Jensen's inequality, we also have 
$
\E\{ \E\{X |X_t\}^2 \} \leq \E\{\E\{X^2 |X_t\}  \}=1.
$
Therefore, $\E\{ \E\{X |X_t\}^2 \}=1$ only if $X$ is a deterministic function of $X_t= \mu_{X,t} X + \sigma_{X,t}W$. But this is impossible when $\sigma_{X,t} >0$. Therefore $\sigma_{U,t} >0$, and $\gamma_t$ in \eqref{eq:olDelmusig_bnds} is strictly positive.

From \eqref{eq:ht_lip}, we obtain
\beq
\begin{split}
& \E\{ G  h_t(\mu_{\tU, T+t} G + \sigma_{\tU ,T+t} W_{U,t}; \, Y)  \}
    - \olL_t \E\{  \abs{\olDel_{\mu,t}} G^2 + \abs{\olDel_{\sigma,t}}\cdot \abs{G}\cdot\abs{ W_{U,t}} \} \\
    &  \quad  \leq \E\{ G  h_t(\mu_{U, t} G + \sigma_{U ,t} W_{U,t}; \, Y) \} \\
   &  \quad \leq   \E\{ G  h_t(\mu_{\tU, T+t} G + \sigma_{\tU ,T+t} W_{U,t}; \, Y) \}
    + \olL_t \E\{  \abs{\olDel_{\mu,t}} G^2 + \abs{\olDel_{\sigma,t}}\cdot \abs{G}\cdot\abs{ W_{U,t}} \}.
\end{split}
\label{eq:Ght_bnds2}
\eeq
Now, using  \eqref{eq:SE_def} and \eqref{eq:SE_def_til}, we have:
\begin{equation}
    \begin{split}
    \frac{1}{\sqrt{\delta}} \abs{\mu_{\tX,T+t+1} - \mu_{X,t+1}} & = 
         \Big \vert \E\{G(h_t(\tU_{T+t}; \, Y) - h_t(U_t; \, Y) ) \} \\
         & \hspace{-3em}  - \mu_{U,t} \big( \E\{ h'_t(\tU_{T+t}; Y) \} - \E\{ h'_t(U_t; Y) \} \big) \, - \, \E\{ h_t'(\tU_{T+t}; Y) \}(\mu_{\tU,T+t} - \mu_{U,t}) \Big \vert \\
        & \hspace{-5em}\leq 
          \olL_t \big( \abs{\olDel_{\mu,t}} + \abs{\olDel_{\sigma,t}} (2/\pi) \big)  \, + \, \abs{\mu_{U,t}}\cdot  \lvert \E\{ h_t'(\tU_{T+t}; Y) \} - \E\{ h_t'(U_t; Y) \} \rvert \, + \, \olL_t \abs{\olDel_{\mu,t}}.
    \end{split}
    \label{eq:muXt1_diff2}
\end{equation}
For the inequality above, we used \eqref{eq:Ght_bnds2} (noting that $\E\{ \abs{W_{U,t}} \} =\E\{ \abs{G} \} = \sqrt{2/\pi}$ and $\E\{ G^2 \} =1$), and the fact that $\lvert h_t' \rvert$ is bounded by $\olL_t$,  the Lipschitz constant of $h_t$. Now, 
\beq
\begin{split}
& \abs{ \E\{ h_t'(U_t; Y)\ - \E\{ h_t'(\tU_{T+t}; Y) \} } \\
&\hspace{4em}= \abs{\E\{ h_t'( \mu_{U,t} G + \, \sigma_{U,t} W_{U,t} ; Y) \} - \E\{ h_t'(\mu_{\tU,T+t} G + \, \sigma_{\tU,T+t} W_{U,t}\, ; Y) \}}.
\end{split}
\eeq
By the induction hypothesis \eqref{eq:SE_gapU}, we have 
\beq
\lim_{T \to \infty} \mu_{\tU, T+t} = \mu_{U, t}, \qquad \lim_{T \to \infty} \sigma_{\tU, T+t} = \sigma_{U, t}.
\label{eq:musig_limits}
\eeq
Thus, as $T \to \infty$, the random variable $(\mu_{\tU,T+t} G + \, \sigma_{\tU,T+t} W_{U,t})$ converges in distribution to $\mu_{U,t} G + \, \sigma_{U,t} W_{U,t}$. Then, Lemma \ref{lem:lipderiv} in Appendix \ref{app:aux_Lip_lemma} implies that
\beq
\lim_{T \to \infty} \,  \abs{ \E\{ h_t'(U_t; Y)\ - \E\{ h_t'(\tU_{T+t}; Y) \} } = 0.
\label{eq:htpr_bnd}
\eeq
Using \eqref{eq:htpr_bnd}, \eqref{eq:olDelmusig_bnds} and \eqref{eq:bdfinite} in \eqref{eq:muXt1_diff2} proves that the first limit in \eqref{eq:SE_gapXt1} holds. 

Finally, we prove the second limit in \eqref{eq:SE_gapXt1}. From \eqref{eq:SE_def}, \eqref{eq:SE_def_til} and arguments along the same lines as \eqref{eq:absft_LB2}-\eqref{eq:ft2_bnd}, we obtain the bound
\beq
\begin{split}\label{eq:fub}
   &  \abs{ \sigma^2_{X,t+1} - \sigma^2_{\tX,T+ t+1}}  = \abs{ \E\{ h_t(U_t; Y)^2 \} \, - \, 
   \E\{ h_t(\tU_{t+T}; Y)^2 \}} \\
   & \leq 2\olL_t^2 (\abs{\olDel_{\mu, t}}^2 + \abs{\olDel_{\sigma, t}}^2) 
+ 2\olL_t \sqrt{ (\sigma_{X,t+1}^2+\sigma^2_{\tX,T+ t+1})(\abs{\olDel_{\mu, t}}^2 + \abs{\olDel_{\sigma, t}}^2)}.
\end{split}
\eeq
Furthermore, as $h_t$ is Lipschitz, the formulas for $\sigma_{X,t+1}^2$ and $\sigma_{\tX,T+t+1}$ (in \eqref{eq:SE_def} and \eqref{eq:SE_def_til}) along with the induction hypothesis \eqref{eq:musig_limits} imply that
\begin{equation}
    \sigma_{X,t+1}^2+\sigma^2_{\tX,T+ t+1}\le c_t,\label{eq:bdfinite2}
\end{equation}
for some constant $c_t$. By using \eqref{eq:bdfinite2} and \eqref{eq:olDelmusig_bnds}, we can upper bound the RHS of \eqref{eq:fub} with $  \kappa_{t+1} \eps$, for sufficiently large $T$. This completes the proof of the second limit in \eqref{eq:SE_gapXt1}. 

\vspace{1em}

\noindent \textbf{Proof of \eqref{eq:uhat_til} and \eqref{eq:xhat_til}.}

Since $\psi \in \PL(2)$, for $i \in [d]$ we have
\beq
 \abs{ \psi(x_i, \tx^{T+t+1}_i) - \psi(x_i, \hx^{t+1}_i)}
\leq C\left(1+ \abs{x_i} + | \tx^{T+t+1}_i | + |\hx^{t+1}_i| \right) \abs{\tx^{T+t+1}_i - \hx^{t+1}_i },
\eeq
for a universal constant $C >0$. Therefore, 
\beq
\begin{split}
& \abs{ \frac{1}{d} \sum_{i=1}^d \psi(x_i, \tx^{T+t+1}_i)  - \frac{1}{d} \sum_{i=1}^d \psi(x_i, \hx^{t+1}_i) }  \leq   \frac{C}{d} \sum_{i=1}^d \left(1+ \abs{x_i} + | \tx^{T+t+1}_i | + |\hx^{t+1}_i| \right) \abs{\tx^{T+t+1}_i - \hx^{t+1}_i } \\
&  \quad \leq  4C \Bigg[ 1+ \frac{1}{d} \sum_{i=1}^d \Big( \abs{x_i}^{2} + | \tx^{T+t+1}_i |^{2} + |\hx^{t+1}_i|^{2}\Big) \Bigg]^{1/2} \,  \frac{\| \tbx^{T+t+1}  - \hbx^{t+1} \|_2}{\sqrt{d}},
\end{split}
\label{eq:psi_xtil_xhat}
\eeq
where  the second inequality follows from Cauchy-Schwarz. By the same argument,
\beq
\begin{split}
  &   \abs{  \frac{1}{n} \sum_{i=1}^n \psi(y_i, \tu^{T+t}_i)  - \frac{1}{n} \sum_{i=1}^n \psi(y_i, \hu^t_i) } 
     \leq 4C \Bigg[ 1+ \frac{1}{n} \sum_{i=1}^n \Big( \abs{y_i}^{2} + | \tu^{T+t}_i |^{2} + |\hu^t_i|^{2}\Big) \Bigg]^{\frac{1}{2}} \,  \frac{\| \tbu^{T+t}  - \hbu^t \|_2}{\sqrt{n}}.
\end{split}
\label{eq:psi_util_uhat}
\eeq
We will show via induction that as $d \to \infty$: \emph{i)} the terms inside the square brackets in \eqref{eq:psi_xtil_xhat} and \eqref{eq:psi_util_uhat} converge almost surely to finite deterministic values, and  \emph{ii)} as $T\to\infty$ (with the limit in $T$ taken after the limit in $d$), the terms $\frac{\| \tbx^{T+t}  - \hbx^t \|_2}{\sqrt{d}}$ and $\frac{\| \tbu^{T+t+1}  - \hbu^{t+1} \|_2}{\sqrt{d}}$ converge to $0$ almost surely.

\underline{Base case $t=0$}: The result \eqref{eq:xhat_til} for $t+1=0$ directly follows from Lemma \ref{lem:xsxT}. Next,  using \eqref{eq:psi_util_uhat}, the LHS  of \eqref{eq:uhat_til} for $t=0$ can be bounded as
\beq
\begin{split}
  &   \abs{  \frac{1}{n} \sum_{i=1}^n \psi(y_i, \tu^{T}_i)  - \frac{1}{n} \sum_{i=1}^n \psi(y_i, \hu^0_i) } 
     \leq 4C \Bigg[ 1+ \frac{\| \by \|_2^2}{n}  + \frac{\| \tbu^T \|_2^2}{n}
    + \frac{\| \hbu^0 \|_2^2}{n} \Bigg]^{\frac{1}{2}} \,  \frac{\| \tbu^{T}  - \hbu^0 \|_2}{\sqrt{n}}.
\end{split}
\label{eq:psi_util_uhat_0}
\eeq
From the definition of  the artificial  GAMP \eqref{eq:txk_update}-\eqref{eq:tftg_geT}, we have 
\beq
    \tbu^T = \frac{1}{\sqrt{\delta}} \bA f_0(\tbx^T) -  \sqrt{\delta} \, \tsb_T \,  \bZ \tbu^{T-1},
\label{eq:uTexp}
\eeq
where we define 
\begin{equation}\label{eq:defZproof}
  \bZ=\bZ_s(\lambda_\delta^*\bI-\bZ_s)^{-1},  
\end{equation}
with $\bZ_s = \diag(\cT_s(y_1), \ldots, \cT_s(y_n) )$. Similarly, defining
\begin{equation}
    \be_1 := \tbu^{T-1} - \tbu^{T-2},
    \label{eq:e1_def}
\end{equation}
we obtain $\tbu^{T-1} = \frac{1}{\sqrt{\delta} \beta_{T-1}} \left[ \bA \tbx^{T-1} - \bZ \tbu^{T-1} + \bZ\be_1 \right]$, or 
\beq
\tbu^{T-1} =  \frac{1}{\sqrt{\delta} \beta_{T-1}} \Bigg( \bI + \frac{1}{\sqrt{\delta} \beta_{T-1}}\bZ \Bigg)^{-1}  \left[ \bA \tbx^{T-1} + \bZ\be_1  \right].
\label{eq:uT1exp}
\eeq
Substituting \eqref{eq:uT1exp} in \eqref{eq:uTexp}, we obtain
\begin{equation}
    \tbu^T = \frac{1}{\sqrt{\delta}} \bA f_0(\tbx^T) - \frac{\tsb_T}{\beta_{T-1}} \bZ
    \Bigg( \bI + \frac{1}{\sqrt{\delta} \beta_{T-1}}\bZ \Bigg)^{-1} \bA \tbx^{T-1}
    - \frac{\tsb_T}{\beta_{T-1}}  \bZ^2 \Bigg( \bI + \frac{1}{\sqrt{\delta} \beta_{T-1}}\bZ \Bigg)^{-1} \be_1.
    \label{eq:tuT_exp}
\end{equation}
Using  \eqref{eq:tuT_exp} and the expression for $\hbu^0$ from \eqref{eq:u0_spec_mod}, we have
\begin{align}
    \frac{1}{d} \| \tbu^T - \hbu^0 \|_2^2  & \leq 3\,  \frac{\| \bA f_0(\tbx^T) -\bA f_0(\hbx^0) \|_2^2}{\delta \, d }  \, + 3\,  \norm{ \frac{\tsb_T}{\beta_{T-1}}  \bZ^2 \Bigg( \bI + \frac{1}{\sqrt{\delta} \beta_{T-1}}\bZ \Bigg)^{-1} \frac{\be_1}{\sqrt{d}} }_2^2 \nonumber \\ 
    & \quad \, + \,  \frac{3}{d} \norm{  \frac{\bsb_0 \sqrt{\delta}}{\lambda_\delta^*} \bZ_s \bA \hbx^0 \, - \,  \frac{\tsb_T}{\beta_{T-1}} \bZ \Bigg( \bI + \frac{1}{\sqrt{\delta} \beta_{T-1}}\bZ \Bigg)^{-1} \bA \tbx^{T-1} }_2^2  \nonumber \\
    & := 3(S_1 + S_2 + S_{3}).
    \label{eq:S1S2S3_split}
\end{align}
We now bound each of the three terms. By Cauchy-Schwarz inequality,
\beq
\begin{split}
    S_1 & \leq \| \bA \|_{\op}^2 \, \frac{\| f_0(\tbx^T) - f_{0}(\hbx^0) \|_2^2}{\delta\, d} \leq
     \| \bA \|_{\op}^2 \,\frac{L_0^2}{\delta} \cdot  \frac{\| \tbx^T - \hbx^0 \|_2^2}{d},
\end{split}
\eeq
where $L_0$ is the Lipschitz constant of $f_0$.  Since the entries of $\bA$ are i.i.d. $\normal(0, 1/d)$, almost surely the operator norm of $\bA$ is bounded  by a universal constant for sufficiently large $d$ \cite{AGZ}. From Lemma \ref{lem:xsxT} and the definition of $\hbx^0$ in \eqref{eq:x0_spec_mod}, we also have
\begin{align}
    \lim_{T \to \infty} \lim_{d \to  \infty} \, \frac{\| \tbx^T - \hbx^0 \|_2^2}{d}
    = \frac{1}{\delta} \cdot \frac{\| \sqrt{\delta} \tbx^T  - \sqrt{d}\bxs \|_2^2}{d} = 0 \quad 
    \text{a.s.}
\end{align}
Therefore,
\beq
\lim_{T \to \infty} \lim_{d \to \infty} S_1 =0 \  \text{ a.s. }
\label{eq:S10}
\eeq
Next, recalling the definition of $\be_1$ from \eqref{eq:e1_def} we bound $S_2$ as follows:
\begin{equation}
    \begin{split}
        S_2 \leq \frac{\tsb_T^2}{\beta_{T-1}^2}  \| \bZ^2 \big( \bI + \bZ/({\sqrt{\delta} \beta_{T-1}}) \big)^{-1}  \|^2_{\op}
        \cdot \frac{\| \tbu^{T-1} - \tbu^{T-2} \|_2^2}{d}.
    \end{split}
    \label{eq:S2_bnd0}
\end{equation}
From Lemma \ref{lem:diffs},  we know that 
$\lim_{T \to \infty} \lim_{d \to \infty} \, \frac{\| \tbu^{T-1} - \tbu^{T-2} \|_2^2}{d}=0$ almost surely. We now show that the other terms on the RHS of \eqref{eq:S2_bnd0} are bounded almost surely. Recall from \eqref{eq:onsager_art} that $\tsb_T = \frac{1}{n} \sum_{i=1}^d f'_0(\tx_i^T)$.
Proposition \ref{prop:tilSE} guarantees that the empirical distribution of $\tbx^t$ converges to the law of $\tX_t \equiv \mu_{\tX,t} X + \sigma_{\tX,t} W$. Since $f_0$ is Lipschitz, Lemma \ref{lem:lipderiv} in Appendix \ref{app:aux_Lip_lemma} therefore implies that almost surely:
\beq
    \lim_{ d \to \infty} \tsb_T = \frac{1}{\delta} \E\{ f_0'(\mu_{\tX,T} X + \sigma_{\tX,T} W)\}.
\eeq
From Lemma \ref{lem:SE_FP_phase1}, we know that $ \lim_{T \to \infty} \mu_{\tX,T} = \frac{a}{\sqrt{\delta}}$ and 
$\lim_{T \to \infty} \, \sigma_{\tX,T}^2 =\frac{1-a^2}{\delta}$. Therefore, letting $T \to \infty$ and applying Lemma \ref{lem:lipderiv} again, we obtain
\beq
\lim_{T \to \infty} \lim_{ d \to \infty} \tsb_T  = \frac{1}{\delta} \, \E\left\{ f_0'\Big( \frac{a}{\sqrt{\delta}} X  \, +  \, \frac{\sqrt{1-a^2}}{\sqrt{\delta}}  W \Big)  \right\} \ \text{ a.s. }
\label{eq:tsbT_lim}
\eeq
From Lemma \ref{lem:SE_FP_phase1}, we have $\beta_{T-1} \to 1/\sqrt{\delta}$ as $T \to \infty$. Also recall from assumption \textbf{(A2)} on p. \pageref{assump:A2} that  $\tau$ is the supremum of the support of $Z_s$, and that $\lambda_\delta^* > \tau$. Therefore,
$Z = \frac{Z_s}{\lambda^*_\delta -Z_s}$ has bounded support, due to which 
$\| \bZ^2 \big( \bI + \bZ/({\sqrt{\delta} \beta_{T-1}}) \big)^{-1}  \|^2_{\op} < C$
for a universal constant $C >0$. 
 Hence,
\beq
\lim_{T \to \infty} \lim_{d \to \infty} S_2 =0 \  \text{ a.s.}
\label{eq:S20}
\eeq

To bound $S_3$, we first write the term inside the norm on the second line of \eqref{eq:S1S2S3_split} as 
\begin{align*}
    \frac{ \sqrt{\delta}}{\lambda_\delta^*} \bZ_s \bA \hbx^0(\bsb_0 - \tsb_T)   + 
    \frac{\tsb_T}{\lambda_\delta^*} \bZ_s \bA  \left( \sqrt{\delta}\hbx^0 - \frac{\tbx^{T-1}}{\beta_{T-1}} \right)  
     +   \frac{\tsb_T}{\beta_{T-1}} \left( \frac{\bZ_s}{\lambda_\delta^*} - \bZ \Big( \bI + \frac{1}{\sqrt{\delta} \beta_{T-1}}\bZ \Big)^{-1} \right)  \bA \tbx^{T-1}.
\end{align*}
Then, using  triangle inequality and Cauchy-Schwarz,  we have
\begin{equation}
    \begin{split}
        S_3 &  \leq  \frac{3\delta}{(\lambda_\delta^*)^2}\frac{\| \bZ_s \bA \hbx^0 \|_2^2}{d}
        (\bsb_0 - \tsb_T)^2 \, + \, \frac{3\tsb_T^2}{(\lambda_\delta^*)^2} \| \bZ_s  \bA \|_{\op}^2  \frac{\|\sqrt{\delta} \hbx^0 - {\tbx^{T-1}}/\beta_{T-1} \|_2^2}{d} \\
        & \quad + \, \frac{3\tsb_T^2}{\beta_{T-1}^2} \, \frac{\| \bA \tbx^{T-1} \|_2^2}{d}
        \,  \norm{\frac{1}{\lambda_\delta^*}\bZ_s \, - \, \bZ \Bigg( \bI + \frac{1}{\sqrt{\delta} \beta_{T-1}}\bZ \Bigg)^{-1} }_{\op}^2 \ := 3( S_{3a} + S_{3b} + S_{3c}).
    \end{split}
    \label{eq:S3split}
\end{equation}
Using the expression for $\hbx^0$ from \eqref{eq:x0_spec_mod} and applying Cauchy-Schwarz, we can bound $S_{3a}$ as:
\beq
S_{3a} \leq  \frac{1}{(\lambda_\delta^*)^2} \| \bZ_s \|^2_{\op} \| \bA \|_{\op}^2 \| \bxs \|_2^2(\bsb_0 - \tsb_T)^2.
\eeq
We note that $Z_s$ is bounded, $\| \bxs \|_2=1$, and $\| \bA \|_{\op}^2$ is bounded almost surely by a universal constant for sufficiently large $d$. Moreover, recalling the definitions of $\bsb_0$ and $X_0 = \mu_{X,0} X  + \sigma_{X,0}W_{X,0}$ from  \eqref{eq:bar_ons_terms} and \eqref{eq:SE_init}, we see that $\bsb_0 = \frac{1}{\delta} \E\{ f_0'( X_0 )\}$ is the limit of $\tsb_T$ in \eqref{eq:tsbT_lim}. Therefore $\lim_{T \to \infty} \lim_{d \to \infty} S_{3a}=0$ almost surely. 

Next, we bound $S_{3b}$. Recalling that $\hbx^0 = \sqrt{d} \bxs / \sqrt{\delta}$, we have
\beq
\begin{split}
& \frac{\|\sqrt{\delta} \hbx^0 - {\tbx^{T-1}}/\beta_{T-1} \|_2^2}{d}
 = \frac{\| \sqrt{d} \bxs - \sqrt{\delta}\tbx^T  + \sqrt{\delta}\tbx^T 
- \sqrt{\delta}\tbx^{T-1} + \sqrt{\delta}\tbx^{T-1} - \tbx^{T-1}/\beta_{T-1} \|_2^2}{d} \\ 
& \qquad \leq \frac{3\| \sqrt{d} \bxs - \sqrt{\delta}\tbx^T  \|_2^2}{d} 
+ \frac{3\| \sqrt{\delta}\tbx^T  - \sqrt{\delta}\tbx^{T-1}\|_2^2}{d}
+ \frac{3\| \tbx^{T-1} \|_2^2}{d}(\sqrt{\delta} \, - \, 1/\beta_{T-1})^2.
\end{split}
\label{eq:S1b_split}
\eeq
Lemmas  \ref{lem:xsxT} and \ref{lem:diffs} imply that the first two terms on the RHS of \eqref{eq:S1b_split} tend to zero in the iterated limit $T\to \infty, d \to \infty$. Furthermore, from Lemma \ref{lem:SE_FP_phase1}, we have $ \lim_{T \to \infty} \beta_{T-1} = 1/\sqrt{\delta}$. From Proposition \ref{prop:tilSE}, we also  have 
\beq
\lim_{d \to \infty} \frac{\| \tbx^{T-1} \|_2^2}{d} =  \mu_{\tX, T-1}^2 + \sigma^2_{\tX, T-1} = \beta_{T-1}^2 \quad \text{a.s.}
\eeq
Therefore, $\lim_{T \to \infty} \lim_{d \to \infty} S_{3b}=0$ almost surely. 

To bound $S_{3c}$, recalling from \eqref{eq:defZproof} that $Z = \frac{Z_s}{\lambda^*_\delta - Z_s}  $, we have
\begin{equation}
    \frac{1}{\lambda_\delta^*}\bZ_s \, - \, \bZ \Bigg( \bI + \frac{1}{\sqrt{\delta} \beta_{T-1}}\bZ \Bigg)^{-1} = \frac{1}{\beta_{T-1}}
    \bZ_s^2 \left(\lambda_\delta^* \bI + \bZ_s \Big( \frac{1}{\sqrt{\delta} \beta_{T-1}} -1 \Big) \right)^{-1} \frac{(\frac{1}{\sqrt{\delta}}  - \beta_{T-1})}{\lambda_\delta^*}.
\end{equation}
Since $\lim_{T \to \infty} \beta_{T-1} = \frac{1}{\sqrt{\delta}}$, almost surely
\beq
\lim_{T \to \infty} \, \norm{\frac{1}{\lambda_\delta^*}\bZ_s \, - \, \bZ \Bigg( \bI + \frac{1}{\sqrt{\delta} \beta_{T-1}}\bZ \Bigg)^{-1} }_{\op}^2 = 0.
\eeq
Thus  $\lim_{T \to \infty} \lim_{d \to \infty} S_{3c}=0$ almost surely. Using the results above in \eqref{eq:S3split}, we have shown that 
\beq
\lim_{T \to \infty} \lim_{d \to \infty} S_{3}=0 \quad \text{a.s.} 
\label{eq:S30}
\eeq

Using \eqref{eq:S10}, \eqref{eq:S20} and \eqref{eq:S30} in \eqref{eq:S1S2S3_split}, and recalling that $n/d \to \delta$, we obtain
\beq
\lim_{T \to \infty} \lim_{n \to \infty} \, \frac{\| \tbu^T - \hbu^0\|_2}{\sqrt{n}} =0.
\label{eq:uTu0_lim}
\eeq
To complete the proof for the base case, we show that the term inside the brackets in \eqref{eq:psi_util_uhat_0} is finite almost surely as $n \to \infty$.  First, by assumption \textbf{(B2)} on p. \pageref{assump:h_zv},  we have $\lim_{n \to \infty} \| \by \|_2^2/n = \E\{ Y^2\}$ almost surely. Furthermore, by Proposition  \ref{prop:tilSE}, we almost surely have
\beq
\lim_{n \to \infty} \| \tbu^T \|_2^2/n = \mu_{\tU, T}^2 + \sigma_{\tU, T}^2.
\eeq
Next, using the triangle inequality, we have
\beq
\| \tbu^T \|_2 \, -  \, 
 \| \tbu^T  - \hbu^0\|_2 \leq  \| \hbu^0 \|_2 \leq  \| \tbu^T \|_2 \, +  \, 
 \| \tbu^T  - \hbu^0\|_2.
\eeq
Combining this with \eqref{eq:uTu0_lim}, we obtain
\beq
\lim_{T \to \infty} \lim_{n \to \infty} \frac{\| \hbu^0 \|_2^2}{n} = \lim_{T \to \infty}\mu_{\tU, T}^2 + \sigma_{\tU, T}^2 = \mu_{U, 0}^2 + \sigma_{U, 0}^2 \quad \text{a.s.}
\eeq
Therefore, using \eqref{eq:psi_util_uhat_0}, we have shown that
\beq
\lim_{T \to \infty} \lim_{n \to \infty} \,  \abs{  \frac{1}{n} \sum_{i=1}^n \psi(y_i, \tu^{T}_i)  - \frac{1}{n} \sum_{i=1}^n \psi(y_i, \hu^0_i) } = 0 \quad \text{ a.s.}
\eeq

\underline{Induction step}: Assume  that \eqref{eq:uhat_til} holds for some $t$, and that \eqref{eq:xhat_til} holds with $t+1$ replaced by $t$. Also assume towards induction that almost surely
\beq
\lim_{T \to \infty} \lim_{d \to \infty} \frac{\| \tbx^{T+t} - \hbx^t \|_2^2}{d}=0, \qquad \lim_{T \to \infty} \lim_{n \to \infty} \frac{\| \tbu^{T+t} - \hbu^{t} \|_2^2}{n}=0.
\label{eq:sq_diffs}
\eeq
The limits in \eqref{eq:sq_diffs} hold for $t=0$, as established in the proof of the base case  (see \eqref{eq:S1b_split}, \eqref{eq:uTu0_lim}). 

From \eqref{eq:psi_xtil_xhat}, we have the bound
\begin{align}
& \abs{ \frac{1}{d} \sum_{i=1}^d \psi(x_i, \tx^{T+t+1}_i)  - \frac{1}{d} \sum_{i=1}^d \psi(x_i, \hx^{t+1}_i) }    \nonumber  \\
&  \leq  4C \Bigg[ 1+ \frac{\| \bx \|_2^2}{d} + \frac{\| \tbx^{T+t+1} \|_2^2}{d} + \frac{\| \hbx^{t+1} \|_2^2}{d}
\Bigg]^{\frac{1}{2}} \,  \frac{\| \tbx^{T+t+1}  - \hbx^{t+1} \|_2}{\sqrt{d}}. 
\label{eq:psi_xtil_xhat_t1}
\end{align}

Using \eqref{eq:txk_update}, \eqref{eq:tftg_geT}, \eqref{eq:xt_update_mod} and the triangle inequality, we obtain:
\beq
\begin{split}
& \frac{\| \tbx^{T+t+1}  - \hbx^{t+1} \|_2^2}{d}
 \leq  \frac{2}{\delta d} 
\| \bA^{\sT} h_t(\tbu^{T+t}; \by) - \bA^{\sT} h_t(\hbu^{t}; \by)  \|_2^2 \,  + 2\, \frac{\| \tsc_{T+t}f_t(\tbx^{T+t}) - \bsc_t f_t(\hbx^t) \|_2^2}{d} \\
& \leq  \frac{2}{\delta d} 
\| \bA^{\sT} h_t(\tbu^{T+t}; \by) - \bA^{\sT} h_t(\hbu^{t}; \by)  \|_2^2 \,  + 4\, \frac{\| f_t(\tbx^{T+t}) \|_2^2}{d} (\tsc_{T+t} - \bsc_t)^2 \, + 4\, \bsc_t^2 \frac{\| f_t(\tbx^{T+t}) - f_t(\hbx^t) \|_2^2}{d} \,  \\
& :=2 \,S_1 +  4\,S_2 + 4\,S_3.
\end{split}
\label{eq:xTt1_split}
\eeq
The term $S_1$ can be bounded as 
\beq
S_1 \leq \| \bA \|_{\op}^2 \, \frac{\| h_t(\tbu^{T+t}; \by) - h_t(\hbu^{t}; \by)\|_2^2}{\delta d} \leq \| \bA \|_{\op}^2 \, \bar{L}_t^2 \frac{\|\tbu^{T+t} -  \hbu^t \|_2^2}{\delta d},
\eeq
where $\bar{L}_t$ is the Lipschitz constant of the function $h_t$.  Since the operator norm of $\bA$ is bounded almost surely as $d \to \infty$, by the induction hypothesis \eqref{eq:sq_diffs} we have $\lim_{T \to  \infty} \lim_{d \to \infty} \frac{\|\tbu^{T+t} -  \hbu^t \|^2}{\delta d} =0$ almost surely. Therefore, 
\beq
\lim_{T \to \infty} \lim_{d \to \infty} S_1 =0 \quad \text{a.s. }
\label{eq:S1t_step}
\eeq
 
 To bound $S_2$, we recall from \eqref{eq:onsager_art} that 
 $\tsc_{T+t} = \frac{1}{n} \sum_i h'_t(\tu_i^t; y_i)$. Proposition \ref{prop:tilSE} guarantees that the joint empirical distribution of $(\tbu^{T+t}, \by)$ converges to the law of $(\tU_{T+t}, Y) \equiv (\mu_{\tU,T+t} G + \sigma_{\tU,T+t} W_{U,T+t}, \, Y)$. Since $h_t$ is Lipschitz, Lemma \ref{lem:lipderiv} in Appendix \ref{app:aux_Lip_lemma} implies that
 \beq
 \lim_{n \to \infty} \tsc_{T+t} = \E\{ h_t'(\mu_{\tU,T+t} G + \sigma_{\tU,T+t} W_{U,T+t}, \, Y) \} \ \ \text{ a.s.}
 \eeq
 From \eqref{eq:SE_gapU}, we know that $\lim_{T \to \infty} \mu_{\tU,T+t} = \mu_{U,t}$ and $\lim_{T \to \infty} \sigma^2_{\tU,T+t} = \sigma^2_{U,t}$. Therefore  applying Lemma \ref{lem:lipderiv} in Appendix \ref{app:aux_Lip_lemma} again, we obtain:
 \beq
  \lim_{T \to \infty} \lim_{n \to \infty} \tsc_{T+t} = \E\{ h_t'(\mu_{U,t} G + \sigma_{U,t} W_{U,t}, \, Y) \} = \bsc_t \quad \text{ a.s.}
 \label{eq:limtsc_Tt}
 \eeq
 Next, using the result in  Proposition \ref{prop:tilSE} with the test function $\psi(x, \tx) = (f_t(\tx))^2$,  we almost surely have 
 \beq 
 \lim_{T \to \infty} \lim_{d \to \infty}\frac{\| f_t(\tbx^{T+t}) \|_2^2}{d} = \lim_{T\to \infty}\E\{ f_t(\tX_{T+t})^2\} = \E\{ f_t(X_{t})^2\}, 
\eeq 
where the last equality follows from  \eqref{eq:SE_gapU} since $f_t$ is Lipschitz. Combining the above with \eqref{eq:limtsc_Tt}, we obtain 
 \beq
 \lim_{T \to \infty} \lim_{d \to \infty} S_2 = 0 \quad \text{ a.s. }
 \label{eq:S2t_step}
 \eeq
 For the third term $S_3$ in \eqref{eq:xTt1_split}, since $f_t$ is Lipschitz (with Lipschitz constant denoted by $L_t$), we  have the bound:
 \beq
 S_3 \leq  \bsc_t^2 L_t^2 \frac{\| \tbx^{T+t} - \bx^t \|_2^2}{d}. 
 \eeq
Thus, by the induction hypothesis \eqref{eq:sq_diffs}, we obtain
\beq
\lim_{T \to \infty} \lim_{d \to \infty} S_3 =0 \quad \text{ a.s. }
\label{eq:S3t_step}
\eeq
We have therefore shown that 
\beq
\lim_{T\to \infty} \lim_{d \to \infty} \, \frac{\| \tbx^{T+t+1}  - \hbx^{t+1} \|^2}{d} = 0 \quad \text{a.s.}
\label{eq:xTt1_hbt1}
\eeq
Next, we show that the terms inside the brackets on the RHS of \eqref{eq:psi_xtil_xhat_t1} are finite almost surely as $d \to \infty$.
Using the pseudo-Lipschitz test function $\psi(x, \tx) = x^2 + \tx^2$, Proposition \ref{prop:tilSE} implies  that almost surely
\beq
\lim_{d \to \infty} \frac{1}{d} \sum_{i=1}^d  \Big( \abs{x_i}^2 +  | \tx^{T+t+1}_i |^2 \Big) 
= \E \{ X^2 \} + \mu_{\tX, T+t+1}^2 + \sigma_{\tX, T+t+1}^2.
\eeq
Moreover, \eqref{eq:SE_gapX} implies that $\lim_{T \to \infty} \mu_{\tX, T+t+1}^2 + \sigma_{\tX, T+t+1}^2 = \mu_{X,t+1}^2+ \sigma_{X,t+1}^2$. Using the triangle inequality, we have 
\beq 
\| \tbx^{T+t+1}\|_2 - \|  \tbx^{T+t+1} - \hbx^{t+1} \|_2 \leq \| \hbx^{t+1} \|_2 \leq \| \tbx^{T+t+1}\|_2 +
\| \hbx^{t+1} -  \tbx^{T+t+1}  \|_2. 
\label{eq:hbx_tbx_treq}
\eeq
Hence, using \eqref{eq:xTt1_hbt1} and Proposition \ref{prop:tilSE}, we almost surely have  
\beq
\lim_{T \to \infty}  \lim_{d \to \infty} \frac{\| \hbx^{t+1} \|_2^2}{d} = \lim_{T \to \infty} \, 
\lim_{ d \to \infty} \frac{\| \tbx^{T+t+1}\|_2^2}{d} = \lim_{T \to \infty} \Big( \mu_{\tX, T+t+1}^2 + \sigma^2_{\tX,T+ t+1} \Big) = \mu^2_{X, t+1} + \sigma^2_{X,t+1}.
\eeq
We have thus shown via \eqref{eq:psi_xtil_xhat_t1} that almost surely
\beq
\lim_{T \to \infty} \lim_{d \to \infty}  \abs{ \frac{1}{d} \sum_{i=1}^d \psi(x_i, \tx^{T+t+1}_i)  - \frac{1}{d} \sum_{i=1}^d \psi(x_i, \hx^{t+1}_i) } = 0.
\label{eq:psixtil_xhat}
\eeq

To complete the proof via induction, we  need to show that if \eqref{eq:xTt1_hbt1} and \eqref{eq:psixtil_xhat} hold with $(t+1)$ replaced by $t$ for some $t >0$, then almost surely 
\beq
 \lim_{T \to \infty} \lim_{n \to \infty} \frac{\| \tbu^{T+t} - \hbu^{t} \|_2^2}{n}=0,  \qquad 
 \lim_{T \to \infty} \lim_{n \to \infty} \abs{  \frac{1}{n} \sum_{i=1}^n \psi(y_i, \tu^{T+t}_i)  - \frac{1}{n} \sum_{i=1}^n \psi(y_i, \hu^{t}_i)} =0. 
 \label{eq:eq:psi_uTt1_ts}
\eeq
From \eqref{eq:psi_util_uhat}, we have the bound
\begin{align}
& \abs{  \frac{1}{n} \sum_{i=1}^n \psi(y_i, \tu^{T+t}_i)  - \frac{1}{n} \sum_{i=1}^n \psi(y_i, \hu^{t}_i)} \nonumber  \\
& \leq 
4C \Bigg[ 1+ \frac{ \| \by \|_2^2}{n}  + \frac{\|\tbu^{T+t} \|_2^2}{n} + 
\frac{\|\hbu^{t} \|_2^2}{n} \Bigg]^{\frac{1}{2}} \,  \frac{\| \tbu^{T+t}  - \hbu^{t} \|_2}{\sqrt{n}} \, .
\label{eq:psi_uTt1}
\end{align}
Using \eqref{eq:tuk_update}, \eqref{eq:tftg_geT}, \eqref{eq:ut_update_mod} and the triangle inequality, we obtain
\beq
\begin{split}
&\frac{\| \tbu^{T+t}  - \hbu^{t} \|_2^2}{n}
\leq  \frac{2}{\delta n} 
\| \bA f_{t}(\tbx^{T+t}) - \bA f_{t}(\hbx^{t})  \|_2^2   + 2 \frac{\| \tsb_{T+t}h_{t-1}(\tbu^{T+t-1} ; \by) - \bsb_{t} h_{t-1}(\hbu^{t-1} ; \by) \|_2^2}{n} \\
& \leq  \frac{2}{\delta n} 
\| \bA f_{t}(\tbx^{T+t}) - \bA f_{t}(\hbx^{t})  \|_2^2 \,  + 4\, 
\frac{ \| h_{t-1}(\hbu^{t-1} ; \by) \|_2^2}{n} (\tsb_{T+t} - \bsb_{t})^2   \\
& \quad + 4\bsb_{t}^2 \,  \frac{\| h_{t-1}(\tbu^{T+t-1} ; \by) - h_{t-1}(\hbu^{t-1} ; \by) \|_2^2}{n} \,  := 2\,S_1 + 4\, S_2 + 4\,S_3.
\end{split}
\label{eq:uTt1_split}
\eeq
Using arguments along the same lines as \eqref{eq:S1t_step}-\eqref{eq:S3t_step} (omitted for brevity), we can show that almost surely  
$$
\lim_{T \to \infty} \lim_{n \to \infty} S_1 = \lim_{T \to \infty} \lim_{n \to \infty} S_2 = \lim_{T \to \infty} \lim_{n \to \infty} S_3 =0.
$$
Hence $\lim_{T \to \infty} \lim_{n \to \infty} \frac{\| \tbu^{T+t}  - \hbu^{t} \|_2}{\sqrt{n}} =0$ almost surely.
Furthermore, using a triangle inequality argument as in \eqref{eq:hbx_tbx_treq}, we obtain
$\lim_{T \to \infty} \lim_{n \to \infty} \frac{\|\tbu^{T+t} \|_2^2}{n} =
\lim_{T \to \infty} \lim_{n \to \infty} \frac{\|\hbu^{t} \|_2^2}{n}$ almost surely. By Proposition \ref{prop:tilSE} and \eqref{eq:SE_gapU}, the latter limit equals $\mu_{U,t}^2 + \sigma_{U,t}^2$. Using these limits in \eqref{eq:psi_uTt1} yields the result \eqref{eq:eq:psi_uTt1_ts}, and completes the proof of the lemma. \end{proof}

\subsection{Putting Everything Together: Proof of Theorem \ref{thm:GLM_spec}} \label{app:mainproof}
We will first use Lemma \ref{lem:ArtTrueconn} to show that the result of the theorem holds for the GAMP iteration $(\hbx^t, \hbu^t)$, i.e., under the assumptions of Theorem \ref{thm:GLM_spec}, we almost surely have
\begin{align}
 & \lim_{n \to \infty} \frac{1}{n} \sum_{i=1}^n \psi (y_{i},\hu^{t}_i) 
 = \E \left\{ \psi( Y, \,  \mu_{U,t} G   + \sigma_{U,t}  W_{U,t}) \right\}, \quad t \geq 0,
 \label{eq:SE_main_hatU} \\
 & \lim_{d \to \infty} \frac{1}{d}\sum_{i=1}^d \psi (x_{i},\hx^{t+1}_i)  =  \E \left\{ \psi( X, \, \mu_{X, t+1} X   + \sigma_{X, t+1} \, W_{X,t+1} \right\}, \quad  t+1 \geq 0.
\label{eq:SE_main_hat} 
\end{align}
Consider the LHS of \eqref{eq:SE_main_hat}. Using the triangle inequality, for any $T >0$, we have
\beq
\begin{split}
&\left \vert \frac{1}{d} \sum_{i=1}^d \psi (x_{i},\hx^{t+1}_i)  -  \E \left\{ \psi( X, \, \mu_{X,t+1} X   + \sigma_{X,t+1} \, W_{X,t+1} \right\} \right \vert \\
& \leq \Bigg \vert \frac{1}{d} \sum_{i=1}^d \psi (x_{i},\hx^{t+1}_i)  -
\frac{1}{d} \sum_{i=1}^d \psi(x_i, \tx^{T+t+1}_i) \Bigg \vert 
\\
&\hspace{3em}+  \Bigg \vert \frac{1}{d} \sum_{i=1}^d \psi(x_i, \tx^{T+t+1}_i) - \E\{ \psi(X, \mu_{\tX,T+t+1}X + \sigma_{\tX, T+t+1} W_{\tX,T+t+1}) \} \Bigg \vert  \\
&\hspace{3em} + \left \vert  \E\{ \psi(X, \mu_{\tX,T+t+1}X + \sigma_{\tX, T+t+1} W_{\tX,T+t+1}) \} 
- \E\{ \psi(X, \mu_{X,t+1}X + \sigma_{X, t+1} W_{X,t+1}) \} \right \vert \,\\
& :=  T_1 + T_2 + T_3.
\end{split}
\label{eq:T1T2T3_split}
\eeq
 We first bound $T_3$ using the pseudo-Lipschitz property of $\psi$, noting that $W_{\tX,T+t}$ and $W_{X,t}$ are both $\sim \normal(0,1)$:
\beq
\begin{split}
T_3 & \leq \E\left\{  \abs{\psi(X, \mu_{\tX,T+t+1}X + \sigma_{\tX, T+t+1} W) - \psi(X, \mu_{X,t+1}X + \sigma_{X,t+1} W) }\right\}, \qquad W \sim \normal(0,1) \\
&\leq 
C \E\Bigg\{ \Big(1 + \Big[ X^2 + \mu_{\tX,T+t+1}^2 X^2 + \sigma_{\tX, T+t+1}^2 W^2\Big]^{1/2}  \, + \, \Big[ X^2 + \mu_{X,t+1}^2 X^2 + \sigma_{X,t+1}^2 W^2\Big]^{1/2}\Big)  \\
& \qquad \qquad \cdot \Big( X^2(\mu_{\tX,T+t+1} - \mu_{X,t+1})^2 +  
W^2( \sigma_{\tX,T+t+1} - \sigma_{X,t+1})^2 \Big)^{1/2}   \Bigg \} \\
& \leq 3C \left( 3 + \mu_{\tX,T+t+1}^2 + \sigma_{\tX, T+t+1}^2 + \mu_{X,t+1}^2 +  \sigma_{X,t+1}^2  \right)^{1/2} \\
&\hspace{3em}\cdot\left(   (\mu_{\tX,T+t+1} - \mu_{X,t+1})^2 + ( \sigma_{\tX,T+t+1} - \sigma_{X,t+1})^2   \right)^{1/2},
\end{split} 
\eeq
where we have used Cauchy-Schwarz inequality in the last line.  From Lemma \ref{lem:ArtTrueconn} (Eq. \eqref{eq:SE_gapX}), we know that 
$\lim_{T \to \infty} \abs{\mu_{\tX,T+t+1} - \mu_{X,t+1}} =0$ and 
$\lim_{T \to \infty} \abs{\sigma_{\tX,T+t+1} - \sigma_{X,t+1}} =0$. Therefore, 
$\lim_{T \to \infty} T_3 =0$.  Next, from \eqref{eq:xhat_til} we have that $\lim_{T \to \infty} \lim_{d \to \infty} T_1 =0$ almost surely. Furthermore, by Proposition \ref{prop:tilSE}, for any $T >0$ we almost surely have $\lim_{d \to \infty} T_2 =0$. Letting $T, d \to\infty$  (with the limit in $d$ taken first) and noting that the LHS of \eqref{eq:T1T2T3_split} does not depend on $T$, we obtain  that \eqref{eq:SE_main_hat} holds.  

The proof of \eqref{eq:SE_main_hatU} uses a bound similar to \eqref{eq:T1T2T3_split} and arguments along the same lines. It is omitted for brevity.

Next, we prove the main result by showing that under the assumptions of the theorem,  almost surely 
\begin{align}
& \lim_{n \to \infty} \abs{ \frac{1}{n} \sum_{i=1}^n \psi(y_i, u^{t}_i)  - \frac{1}{n} \sum_{i=1}^n \psi(y_i, \hu^{t}_i) } = 0,
\qquad  
\lim_{n \to \infty} \frac{\| \bu^{t} - \hbu^{t} \|_2^2}{n} =0, \quad t \geq 0
\label{eq:psi_limits_U} \\
& \lim_{d \to \infty} \abs{ \frac{1}{d} \sum_{i=1}^d \psi(x_i, x^{t+1}_i)  - \frac{1}{d} \sum_{i=1}^d \psi(x_i, \hx^{t+1}_i) } = 0, \qquad \lim_{d \to \infty}\frac{\| \bx^{t+1} - \hbx^{t+1} \|_2^2}{d}=0, \quad t+1 \geq 0.
\label{eq:psi_limits_X}
\end{align} 
Combining \eqref{eq:psi_limits_X}-\eqref{eq:psi_limits_U} with \eqref{eq:SE_main_hat}-\eqref{eq:SE_main_hatU} yields the results in   \eqref{eq:SE_main} and \eqref{eq:SE_main_U}.

 The proof of \eqref{eq:psi_limits_X} and \eqref{eq:psi_limits_U} is via induction and uses arguments very similar to those to prove \eqref{eq:uhat_til}-\eqref{eq:xhat_til}. To avoid repetition we only provide a few steps. Noting that $\bx^0 = \hbx^0$, we now show \eqref{eq:psi_limits_X}, under the induction hypothesis that \eqref{eq:psi_limits_U} holds and also that \eqref{eq:psi_limits_X} holds with $t+1$ replaced by $t$.

 Since $\psi \in \PL(2)$, we have
\begin{align}
& \abs{ \frac{1}{d} \sum_{i=1}^d \psi(x_i, x^{t+1}_i)  - \frac{1}{d} \sum_{i=1}^d \psi(x_i, \hx^{t+1}_i) }    
  \leq  4C \Bigg[ 1+ \frac{\| \bx \|_2^2}{d} + \frac{\| \bx^{t+1} \|_2^2}{d} + \frac{\| \hbx^{t+1} \|_2^2}{d}
\Bigg]^{\frac{1}{2}} \,  \frac{\| \bx^{t+1}  - \hbx^{t+1} \|_2}{\sqrt{d}}. 
\label{eq:psi_x_xhat_t1}
\end{align}
Furthermore, using the definitions of $\bx^{t+1}$ and $\hbx^{t+1}$, and the triangle inequality we have 
\begin{align}
&    \frac{\| \bx^{t+1}  - \hbx^{t+1} \|_2^2}{d} \leq  \frac{2}{\delta d} 
\| \bA^{\sT} h_{t}(\bu^{t}; \by) - \bA^{\sT} h_{t}(\hbu^{t}; \by)  \|_2^2 \,  +4 \, \frac{\| f_{t}(\bx^{t}) \|_2^2}{d} (\sc_{t} - \bsc_{t})^2  
 \, + \, 4   \bsc_{t}^2 \frac{\| f_{t}(\bx^{t}) - f_{t}(\hbx^{t}) \|_2^2}{d} \nonumber \\
& \leq  \frac{2\bar{L}_{t}^2}{\delta} \| \bA \|_{\op}^2  \frac{\| \bu^{t} - \hbu^{t} \|_2^2}{d} \,  + 4\, \frac{\| f_{t}(\bx^{t}) \|_2^2}{d} (\sc_{t} - \bsc_{t})^2 \, + \, 4\bsc_{t}^2  L_{t}^2\frac{\| \bx^{t} - \hbx^{t} \|_2^2}{d},
\label{eq:x_xh_decomp}
\end{align}
where $L_{t}, \bar{L}_{t}$ are the Lipschitz constants of $f_{t}, h_{t}$, respectively. By the induction hypothesis and Lemma \ref{lem:lipderiv}, the terms $\frac{\| \bu^{t} - \hbu^{t} \|_2^2}{d}$, $\frac{\| \bx^{t} - \hbx^{t} \|_2^2}{d}$, and $(\sc_{t} - \bsc_{t})^2$ tend to zero. Furthermore,  by the induction hypothesis, we almost surely have $\frac{\| f_{t}(\bx^{t}) \|_2^2}{d} \to \E\{ f_{t}(X_{t})^2\}$, and by \eqref{eq:SE_main_hat},  $\frac{\| \hbx^{t+1} \|_2^2}{d} \to  (\mu_{X,t+1}^2 + \sigma_{X,t+1}^2)$ as $d \to \infty$. Finally, by a triangle inequality argument analogous to \eqref{eq:hbx_tbx_treq}, we also have 
$$\lim_{d \to \infty}\frac{\| \bx^{t+1} \|_2^2}{d} = \lim_{d \to \infty}\frac{\| \hbx^{t+1} \|_2^2}{d} = (\mu_{X,t+1}^2 + \sigma_{X,t+1}^2) \quad \text{ a.s.}$$
Using these limits in \eqref{eq:psi_x_xhat_t1} proves \eqref{eq:psi_limits_X}. The proof of  \eqref{eq:psi_limits_U} (under the induction hypothesis that \eqref{eq:psi_limits_X} holds with $(t+1)$ replaced by $t$) is along the same lines: we use a bound similar to \eqref{eq:psi_x_xhat_t1} and a decomposition of $\frac{\| \bu^{t} - \hbu^{t} \|_2^2}{n}$ similar to \eqref{eq:x_xh_decomp}. This completes the proof of the theorem.
\qed

\section{An Auxiliary Lemma} \label{app:aux_Lip_lemma}

The following result is proved in \cite[Lemma 6]{BM-MPCS-2011}.
\begin{lemma}
\label{lem:lipderiv}
Let $F \colon \reals^2 \to\reals$ be a Lipschitz function, and let $F'(u,v)$ denote its derivative with respect to the first argument at $(u,v) \in \reals^2$. Assume that  $F'( \cdot, v)$ is continuous almost everywhere in the first argument, for each $v \in \reals$. Let $(U_m, V_m)$  be a sequence of random vectors in $\reals^2$ converging in distribution to the random vector $(U,V)$ as $m \to \infty$. Furthermore, assume that the distribution of $U$ is absolutely continuous with respect to the Lebesgue measure. Then,
\[ \lim_{m \to \infty}  \E\{ F'(U_m, V_m) \} = \E\{ F'(U,V) \}. \]
\end{lemma}

\section{Complex-valued GAMP} \label{sec:complex_GAMP}

Consider a complex sensing matrix $\bA$ with rows  distributed as $( \ba_i) \ \sim_{\rm i.i.d.} \ \cnormal(0, \bI_d/d))$, for $i \in [n]$.  The output of the GLM $\by  \in \mathbb{C}^n$ is generated as $p_{Y|G}(\by \mid \bg )$, where $\bg = \bA \bx$. The GAMP algorithm for the complex setting has been  studied in the context of phase retrieval by \cite{schniter2014compressive, ma2019optimization}. Here, we briefly review the complex GAMP and present some numerical results for complex GAMP with spectral initialization.

As in Section \ref{sec:simulations},  we take $f_t$ to be the identity function, and $h_t = \sqrt{\delta} h_t^*$, where $h_t^*$ is given in \eqref{eq:opt_gt}. To obtain a compact state evolution recursion, we initialize with a scaled version of the  spectral estimator $\bxs$: 
 \begin{align}
 & \bx^0= \sqrt{d} \, \, \frac{a}{1-a^2}\, \bxs, 
  \qquad \bu^0 = \frac{1}{\sqrt{\delta}} \bA \bx^0 -  \frac{1}{\sqrt{\delta} \lambda_\delta^*} \,  \bZ_s  \bA \bx^0. \label{eq:x0_u0_spec_comp} 
 \end{align}
The iterates are then computed as:
\begin{align}
\bx^{t+1} &= \bA^{\sf H} h_t^*(\bu^t; \by) - \sc_t f_t(\bx^t),  \label{eq:xt_update_comp} \\
\bu^{t+1} &= \frac{1}{\sqrt{\delta}}\bA \bx^{t+1} - \frac{1}{\sqrt{\delta}} h_t^*(\bu^{t}; \by). \label{eq:ut_update_comp} 
\end{align}
Here, the Onsager coefficient $\sc_t$ is given by \cite{schniter2014compressive}:
\beq
\sc_t=  \frac{\sqrt{\delta}}{{\sf Var}(G \mid U_t =u)} \Bigg(\frac{{\sf Var}\{ G \mid U_t= u, \, Y=y \}}{{\sf Var}(G \mid U_t =u)} - 1 \Bigg).
\label{eq:sc_t_comp}
\eeq
For this choice of $f_t, h_t$, the state evolution iteration can be written in terms of a single parameter $\mu_t \equiv \mu_{X,t}$. For $t \geq 0$:
\beq
\begin{split}
& \mu_{U,t} = \frac{1}{\sqrt{\delta}} \mu_t, \qquad \sigma_{U,t}^2 = \frac{\mu_t}{\delta}, \qquad \sigma_{X,t}^2 = \mu_{X,t} = \mu_t,
\\
& \mu_{t+1} = \sqrt{\delta} \E\left\{ \abs{h_t^*(U_t; \, Y)}^2 \right\}.
\end{split}
\label{eq:comp_SE}
\eeq
The recursion is initialized with $\mu_0 = \frac{a^2}{1-a^2}$. Moreover, the parameter $\mu_{t+1}$ can be consistently estimated from the iterate $\bu^t$ as $\hat{\mu}_{t+1} = \sqrt{\delta} \| h^*(\bu^t; \by) \|^2_2/n$. It can also be estimated as the positive solution of the quadratic equation $\hat{\mu}_{t+1}^2 + \hat{\mu}_{t+1} = \| \bx^{t+1}\|^2_2/d$.

\begin{figure}[t!]
    \centering    
    \includegraphics[width=0.65\linewidth]{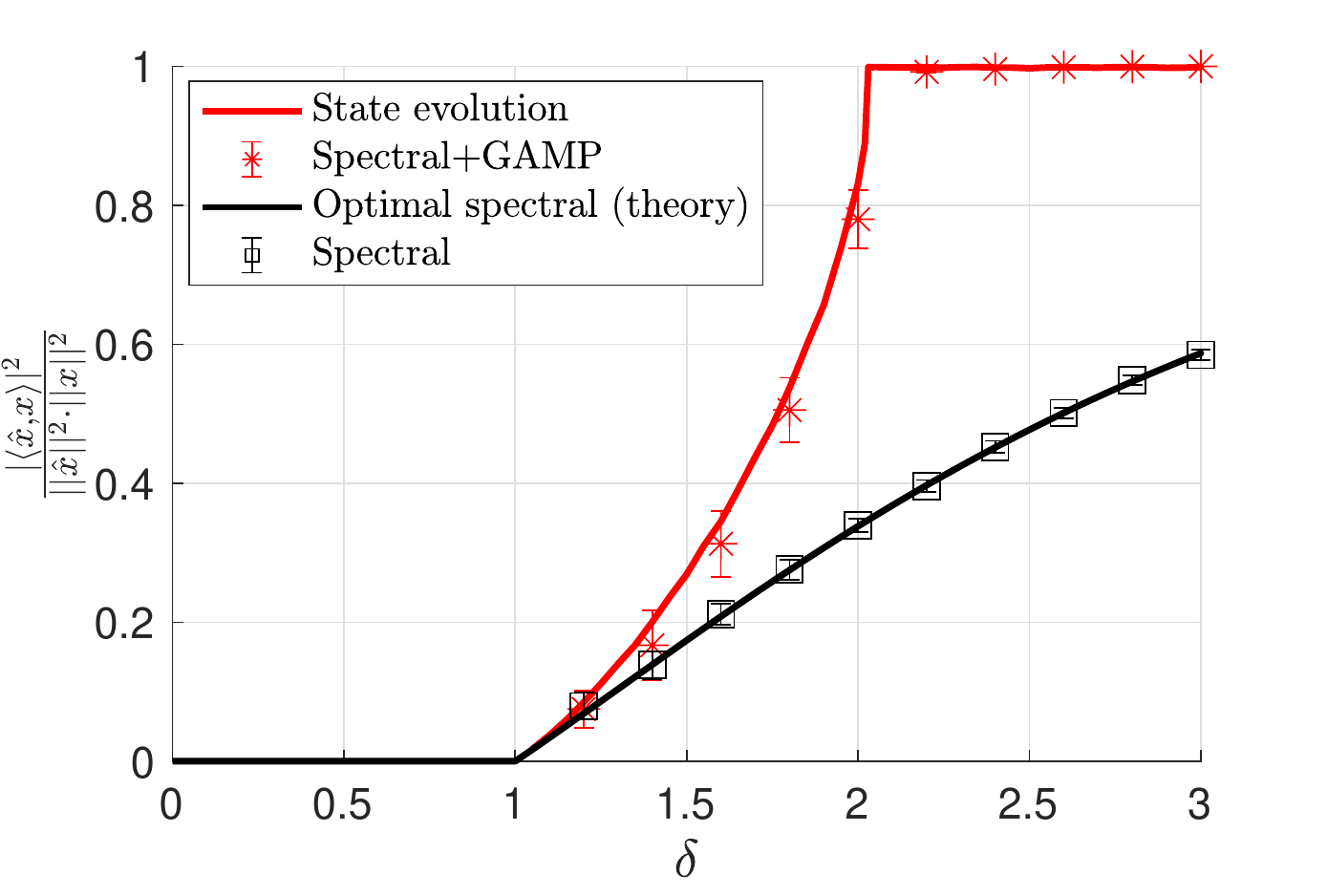}
\caption{Performance comparison between complex GAMP with spectral initialization (in red) and the spectral method alone (in black) for a Gaussian prior $P_X \sim \cnormal(0,1)$. On the $x$-axis, we have the sampling ratio $\delta = n/d$; on the $y$-axis, we have the normalized squared scalar product between the  signal  and the estimate. The experimental results ($*$ and $\square$ markers) are in excellent agreement with the theoretical  predictions (solid lines) given by state evolution for GAMP and Lemma \ref{lemma:pt} for the spectral method. Error bars indicate one standard deviation around the empirical mean.} \label{fig:complex}
\end{figure}

\begin{figure}[t!]
    \centering
    \includegraphics[width=0.65\linewidth]{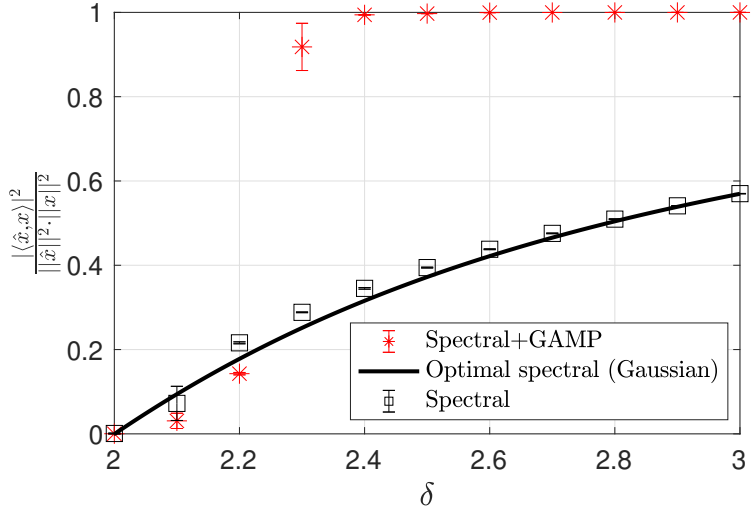}
    \caption{Performance comparison between complex GAMP with spectral initialization (in red) and the spectral method alone (in black) for a model of coded diffraction patterns.}
\label{fig:cdp}
\end{figure}

We now discuss some numerical results for noiseless (complex)  phase retrieval, where $y_i = \abs{(\bA \bx)_i}^2$, for $i \in [n]$.   For a given measurement matrix $\bA$, note that replacing $\bx$ by $e^{i \theta} \bx$ leaves the measurement $\by$ unchanged. Therefore the performance of any estimator is measured up to a constant phase rotation:
\beq 
\label{eq:comp_scalar_prod}
\min_{\theta \in [0, 2 \pi)} \, 
\frac{\Big \vert \langle \hbx, \, e^{i \theta} \bx \rangle \Big \vert^2 }{\| \bx \|^2_2 \, \| \hbx \|^2_2 }. 
\eeq
Figure \ref{fig:complex} shows the performance of GAMP with spectral initialization when the signal $\bx$ is uniform on the $d$-dimensional \emph{complex} sphere with radius $\sqrt{d}$, and the sensing vectors $( \ba_i) \ \sim_{\rm i.i.d.} \ \cnormal(0, \bI_d/d)$.

Figure \ref{fig:cdp} shows the performance  with coded diffraction pattern sensing vectors, given by \eqref{eq:cdp_def}.  The signal $\bx$ is the image in Figure \ref{fig:ven}, which is a $d_1 \times d_2 \times 3$ array with $d_1 = 820$ and $d_2=1280$.  The three components $\bx_j\in \mathbb R^{d}$ ($j\in\{1, 2, 3\}$ and $d=d_1\cdot d_2$) are treated separately, and the performance is measured via the average squared normalized scalar product $\frac{1}{3}\sum_{j=1}^3 \frac{|\langle\hat{\bx}_j, \bx_j\rangle|^2}{\norm{\hat{\bx}_j}_2^2 \norm{\bx_j}_2^2}$. 

The red points in Figure \ref{fig:cdp} are obtained by running the complex GAMP algorithm with spectral initialization, as given in \eqref{eq:x0_u0_spec_comp}-\eqref{eq:sc_t_comp}. We perform $n_{\rm sample} = 5$ independent trials and show error bars at 1 standard deviation. For comparison, the black points correspond to the empirical performance of the spectral method alone, and the black curve gives the theoretical prediction for the optimal squared correlation for Gaussian sensing vectors (see Theorem 1 of \cite{luo2019optimal}).

\end{document}